%% file: arXiv19_sliding.tex
\documentclass[twocolumn, switch]{article} 

\usepackage{preprint}

\usepackage{amsmath, amsthm, amssymb, amsfonts, bm}

\usepackage{natbib}

\usepackage[utf8]{inputenc}	
\usepackage[T1]{fontenc}	
\usepackage{xcolor}		
\usepackage[colorlinks = true,
            linkcolor = purple,
            urlcolor  = blue,
            citecolor = cyan,
            anchorcolor = black]{hyperref}	
\usepackage{booktabs} 		
\usepackage{nicefrac}		
\usepackage{microtype}		
\usepackage{lineno}		
\usepackage{float}			
\usepackage{supertabular}
\usepackage{algorithm,algpseudocode}
\usepackage[inline]{enumitem}

\usepackage{lipsum}		

\usepackage{newfloat}
\DeclareFloatingEnvironment[name={Supplementary Figure}]{suppfigure}
\usepackage{sidecap}
\sidecaptionvpos{figure}{c}

\usepackage{titlesec}
\titlespacing\section{0pt}{12pt plus 3pt minus 3pt}{1pt plus 1pt minus 1pt}
\titlespacing\subsection{0pt}{10pt plus 3pt minus 3pt}{1pt plus 1pt minus 1pt}
\titlespacing\subsubsection{0pt}{8pt plus 3pt minus 3pt}{1pt plus 1pt minus 1pt}

\newcommand{\abs}[1]{\ensuremath{\left\lvert#1\right\rvert}}
\newcommand{\norm}[1]{\left\lVert #1 \right\rVert} 
\newcommand{\mx}[1]{\mathbf{\bm{#1}}} 
\newcommand{\vc}[1]{\mathbf{\bm{#1}}} 
\newcommand{\RN}[1]{\uppercase\expandafter{\romannumeral #1\relax}}
\newcommand{\pder}[2]{\frac{\partial#1}{\partial#2}}

\newcommand{\dA}{\,\mathrm{d}A}

\DeclareMathOperator{\sgn}{sign}
\DeclareMathOperator{\diag}{diag}

\DeclareMathOperator{\Atan2}{Atan2}

\DeclareMathOperator{\sat}{sat}

\newtheorem{thm}{Theorem}

\newtheorem{prop}[thm]{Proposition}
\newtheorem{cor}{Corollary}

\title{Quasi-static Analysis of Planar Sliding Using Friction Patches}

\usepackage{xwatermark}
\newwatermark[firstpage,color=gray!90,angle=0,scale=0.28, xpos=0in,ypos=-5in]{*correspondence: \texttt{mahdi.ghazaei@control.lth.se}}

\usepackage{authblk}

\author[1\thanks{\tt{mahdi.ghazaei@control.lth.se}}]{M. Mahdi Ghazaei Ardakani}
\author[2]{Joao Bimbo}
\author[3]{Domenico Prattichizzo}
\affil[1,2,3]{Istituto Italiano di Tecnologia (IIT), Italy}
\affil[3]{Department of Information Engineering, University of Siena, Italy}

\begin{document}

\twocolumn[ 
  \begin{@twocolumnfalse} 
  
\maketitle

\begin{abstract}
Planar sliding of objects is modeled and analyzed. The model can be used for non-prehensile manipulation of objects lying on a surface. We study possible motions generated by frictional contacts, such as those arising between a soft finger and a flat object on a table. Specifically, using a quasi-static analysis we are able to derive a hybrid dynamical system to predict the motion of the object. The model can be used to find fixed-points of the system and the path taken by the object to reach such configurations. Important information for planning, such as the conditions in which the object sticks to the friction patch, pivots, or completely slides against it are obtained. Experimental results confirm the validity of the model for a wide range of applications.
\end{abstract}
\vspace{0.35cm}

  \end{@twocolumnfalse} 
] 



\section{Introduction}\label{sec:intro}

\input{01_intro}

\section{Modelling}\label{sec:slidingModel}

\input{02_sliding}

\section{Dynamical system}\label{sec:dynSys}

\input{03_dynSys}

\section{Properties of the solution}

\input{04_solProp}

\section{Approximate solution}\label{sec:approxSol}

\input{05_approxSol}

\section{Strategies for sliding}

\input{06_movePrim}

\section{Experiments and results}

\input{07_results}

\section{Discussion}

\input{08_disc}


\section{Conclusion}

\input{09_conclusion}



\footnotesize
\section*{Acknowledgements}
The research has received funding from the SOMA project (European Commission, Horizon 2020 Framework Programme, H2020-ICT-645599).

\normalsize
\bibliography{references}



\appendix
\input{10_appendix}

\end{document}

%% file: 01_intro.tex
An object can be manipulated using either a prehensile or a non-prehensile approach. Thin flat objects on a surface are hard to grasp, but can be manipulated by pushing or pulling. Several strategies are imaginable. If the object is small compared to a hand, it is possible to cage it within the hand. Depending on the height of the object, it could also be pushed from the side by some parts of the hand or moved by using an elevated edge or any bumps or dents on its surface. In certain cases, force closure can be achieved by pressing the object for example against the table and moving the object as if it was grasped. Nevertheless, a practically interesting case is when the hand is placed on top of the object, but the friction between the object and the hand is controlled such that the object can pivot (Fig.~\ref{fig:slidingCellphone}). This strategy is an example of exploiting environmental constraints for manipulation~\citep{malvezzi2019eval}. The benefits are immediate when force closure is  impossible, no matter how hard the object is pressed (e.g., due to a small contact area), or when hand reorientation is limited, for example due to kinematic limitations of the robotic arm.

When a robot end-effector (e.g., a hand or a soft finger) establishes a frictional contact with an object, it can transfer forces through the friction patch formed between them. Such contact, when used for pivoting, behaves similarly to a joint. However, it can transfer not only forces between the end-effector and the object but also torque. This fact can be used to control the angular velocity of the object. 

\begin{figure}
	\centering
	\includegraphics[width=0.8\linewidth]{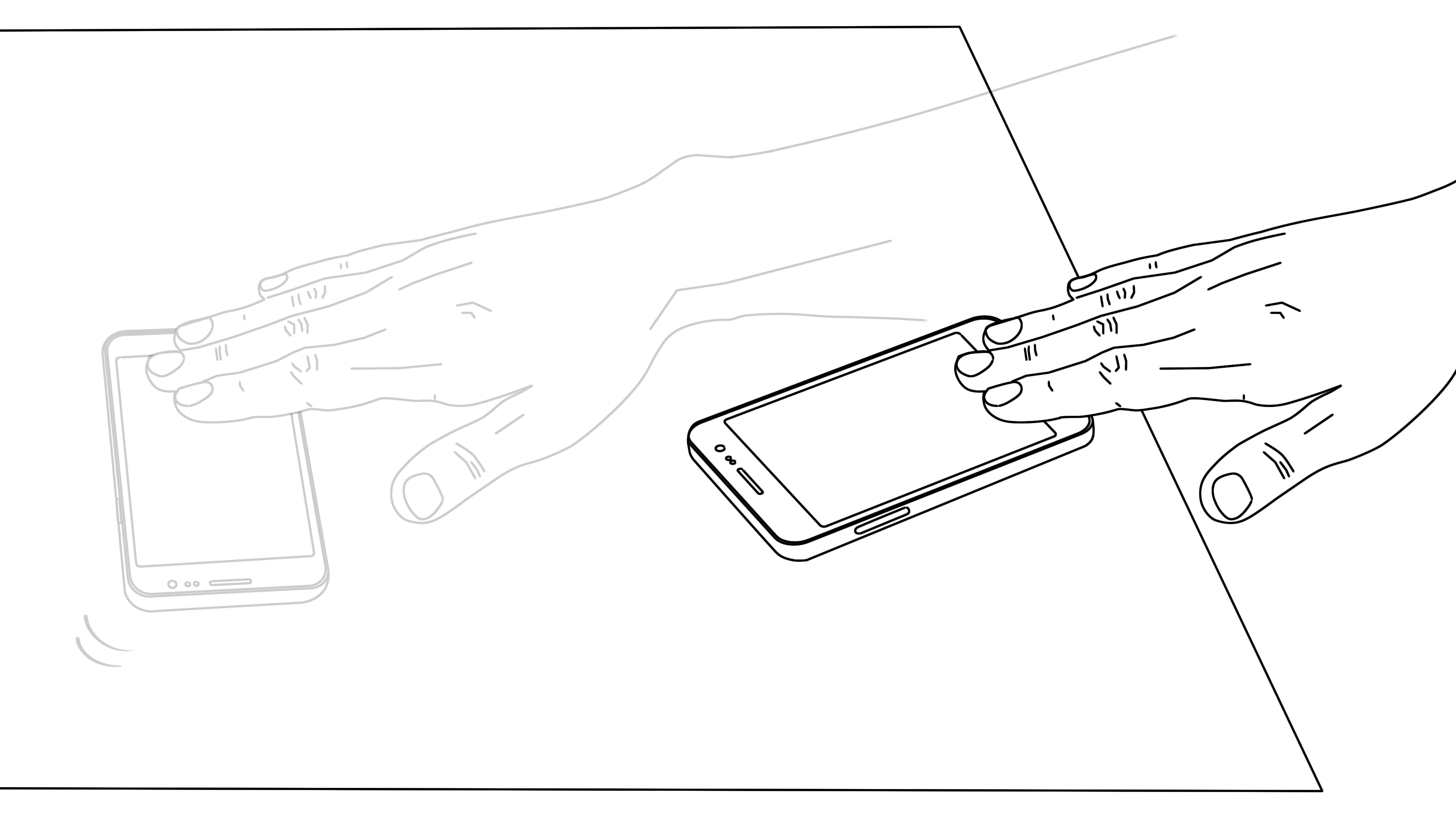}
	\caption{An example in which the friction is utilized to move and reorient the cell phone.}
	\label{fig:slidingCellphone}
\end{figure} 

Contact mechanics is a complex phenomenon~\citep{johnson_1985}.
The concept of limit surfaces has been introduced as a convenient way to characterize the friction properties of two surfaces sliding against each other~\citep{goyal1989planar}. A practical approximation of the limit surface was proposed by~\cite{howe1996practical}. The observation is that limit surfaces can be approximated by ellipsoids. This result has also been confirmed for soft finger models~\citep{xydas1999modeling,fakhari2016development}.

Planar sliding by means of pushing from a side of an object is a classic topic in robotics~\citep{mason1982manipulator,peshkin1987planning}. Active sensing and pushing using only tactile feedback was proposed by~\cite{lynch1992manipulation}. \cite{ruiz2011fast} studied pushing objects with rectangular support. Under uncertainty of the center of pressure,~\cite{huang2017exact} provides exact bounds for the motion of a sliding object. Pushing on the edge of blocks is revisited based on differential flatness for planning and control~\citep{zhou2017pushing}. In these scenarios, the contact surface between the pusher and the object has been assumed small such that no moment can be transferred through the contact area.

Exploiting extrinsic dexterity, a grasped object can be manipulated. For example, \cite{chavan2017stable,chavan-dafle2018rss} have considered scenarios where objects were pushed against various fixtures, while they were ensured to stick to the fixtures. A discrete set of hard point contacts was used to model the friction between the fixture and the object, while for the grasp the limit surface concept was utilized. For planning of stable pushes, a polyhedral approximation to motion cones was calculated~\citep{chavan2017stable}.

Using friction patches has also been studied in the context of in-hand manipulation tasks~\citep{bicchi1993experimental,shi2017dynamic}. For the task of dynamic pivoting, assuming a ``pivoting joint'' and using a simple friction model, a robust controller was proposed to cope with the uncertainty in the torsional friction~\citep{hou2016robust}. Additionally, adaptive control strategies have been considered by~\cite{karay2016adapt}. Common simplifying assumptions used also in these works are that the contact points of the fingers and the object are fixed and/or the friction is isotropic with limit surfaces described by diagonal matrices.

A generic manipulation problem with compliance and sliding was studied by~\cite{kao1992quasistatic}. In the subsequent articles,~\cite{kao1993comparison,xue1994dexterous} analyzed manipulation of a business card on a frictionless table top using symmetric motion of two soft fingers. A good match between the theory and the experiments was reported.


The problem of planar sliding when there is torque transfer between the object and the manipulator, due to the dimension of the contact area, can be regarded as an extension of pushing problems~\citep{lynch1992manipulation,zhou2017fast}. Some of the previously suggested approaches can also be adapted to this problem~\citep{kao1992quasistatic,shi2017dynamic,chavan-dafle2018rss}. Nevertheless, there has not yet been a complete analysis of this problem per se. Thus, the aim of this article is to provide an adequate mathematical model of planar sliding using friction patches for the purpose of control and planning.

%% file: 02_sliding.tex
Consider the configuration shown in Figure.~\ref{fig:diagSliding}. We refer to the part moving the object as a hand, which can be any part of a robotic hand such as a soft finger or in general a part of an end-effector. We assume that the object is rigid and the hand does not roll against it. The patch, which is the part of the hand in touch with the object, may however deform. Accordingly, we assign body-fixed frames to the object and to the patch. The frame attached to the patch is designated by $\{\mathrm{H}\}$ and the frame of the object by $\{\mathrm{O}\}$. For convenience, we consider the frame of the object to be fixed at its center of mass (COM) with its $z$-axis orthogonal to the sliding surface. For the hand, we consider the frame to be fixed at the centroid of the patch. Moreover, we assume that the limit surfaces (LS) are symmetric with respect to the origin (if the direction of motion is reversed, so are friction forces). They can also be approximated by ellipsoids with respect to the Center of Pressure (COP).

\begin{figure}
	\centering
	\includegraphics[width=0.8\linewidth]{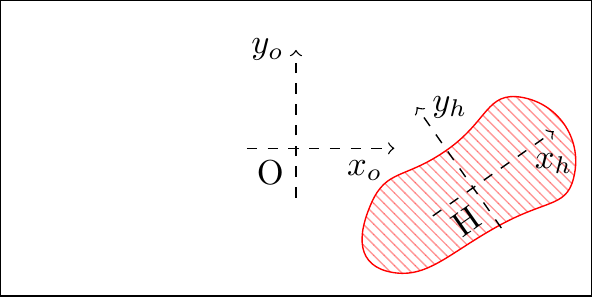}
	\caption{The coordinate frames attached to the object \{O\} and the patch \{H\}.}
	\label{fig:diagSliding}
\end{figure} 

Here is the list of the assumptions:
\begin{itemize}
    \item The surface and the object are much more rigid compared to the robotic hand.
    \item The hand is placed on top of a flat object.
    \item The hand does not roll against the object.
    \item The Coulomb friction model is used.
    \item The limit surfaces are symmetric with respect to the origin and can be approximated by ellipsoids.
    \item The velocities are low enough such that the quasi-static condition holds.
\end{itemize}

\subsection{Summary of nomenclature}
 Capital bold letters denote matrices. Vectors are denoted by an arrow above a symbol, while small bold letters represent coordinate vectors. Scalars are typeset in roman.
 
\begin{center}
\begin{supertabular}{ll}
$\vc{q}_o$	& generalized coordinates of the object \\ 
$\vc{q}_h$	& generalized coordinates of the patch \\ 
$\vc{\nu}_o$ 	& twist of the object wrt frame $\{\mathrm{O}\}$  \\
$\vc{\nu}_h$ 	& twist of the patch wrt frame $\{\mathrm{H}\}$  \\
$\vc{v}_h$ 	& linear velocity of the patch\\
$\vc{\nu}_{rel}$ 	& relative twist of the patch and the object in $\{\mathrm{H}\}$\\
$\vc{w}_o$	& wrenches exerted on the object wrt frame $\{\mathrm{O}\}$\\
$\vc{w}_h$	& wrenches exerted on the hand through the patch\\ 
$\vc{p}$		& position of the pivot point wrt frame $\{\mathrm{H}\}$\\ 
$\vc{m}$		& torque \\
$\mx{J}(\vc{r})$	& Jacobian for a point at relative coordinates $\vc{r}$ \\
$\mx{R}(\theta)$	& Rotation matrix along $z$ direction with $\theta\,$rad\\
$\mx{G}(\vc{q})$ & Jacobian for a frame with relative coordinates $\vc{q}$\\
$\mx{A}$ 		& LS of object-surface contact wrt frame $\{\mathrm{O}\}$\\
$\mx{B}$		& LS of hand-object contact wrt frame $\{\mathrm{H}\}$\\
$\hat{\mx{A}}$ 	& LS of object-surface contact wrt frame $\{\mathrm{H}\}$ \\
$\mx{\Phi}$		& principal sliding wrenches \\
$\mx{\Lambda}$		& generalized eigenvalues of $\hat{\mx{A}}$ and $\mx{B}$\\
\end{supertabular}
\end{center}

\subsection{Preliminaries}
 
The generalized coordinates of the object are denoted by
\begin{equation*}
\vc{q}_o = \left[x_o,\, y_o,\, \theta_o \right]^T
\end{equation*}
and the twist and the wrenches expressed in the body-fixed frame are
\begin{align*}
\vc{\nu}_o &= \left[v_{xo},\, v_{yo},\, \omega_o \right]^T, \\
\vc{w}_o &= \left[f_{xo},\, f_{yo},\, m_o \right]^T.
\end{align*}
Similarly, the respective quantities for the friction patch are defined and denoted by the subscript $h$.

The relative coordinates of frame $\{\mathrm{H}\}$ with respect to $\{\mathrm{O}\}$ can be written as
\begin{align*}
\vc{q}_{rel} &:= \left[x_{r},\, y_{r},\, \theta_{r} \right]^T \\ 
&\phantom{:}=\mx{R}(-\theta_o) \left(\vc{q}_{h} - \vc{q}_{o}\right)
\end{align*}
where $\mx{R}(\cdot)$ is the rotation matrix along the $z$-axis
\begin{align} \label{eq:rotz}
\mx{R}(\theta):=
	\begin{bmatrix}
		\cos(\theta)	& -\sin(\theta) & 0 \\
		\sin(\theta)	& \cos(\theta) 	& 0 \\
		0 				& 0				& 1
	\end{bmatrix}.
\end{align}

We also define for a position vector $\vc{r} = [x_r,\, y_r]^T$,
\begin{align} \label{eq:jac}
\mx{J}(\vc{r}) := 
	\begin{bmatrix}
		1 & 0 & -y_r \\
		0 & 1 &  x_r \\
		0 & 0 &  1
	\end{bmatrix}
\end{align}
with the property that $\mx{J}(\vc{r}_1)\mx{J}(\vc{r}_2) = \mx{J}(\vc{r}_1 + \vc{r}_2)$ for any $\vc{r}_1$ and $\vc{r}_2$.
Accordingly, 
\begin{align*}
\mx{J}^{-1}(\vc{r}) = \mx{J}(-\vc{r}) =
	\begin{bmatrix}
		1 & 0 &  y_r \\
		0 & 1 &  -x_r \\
		0 & 0 &  1
	\end{bmatrix}.
\end{align*}

\begin{prop}\label{thm:trans}
The relation between planar twists $\vc{\nu}_p$ given in frame $\{\mathrm{P}\}$ and $\vc{\nu}_o$ given in frame $\{\mathrm{O}\}$, with relative coordinates $\vc{q}_{rel} = \left[\vc{r},\, \theta_{r} \right]^T$ is
\begin{equation}
	\vc{\nu}_p= \mx{G} \vc{\nu}_o
\end{equation}
where 
\begin{equation}
\mx{G}: = \mx{G}(\vc{q}_{rel}) = \mx{R}^T(\theta_r) \mx{J}\left( \vc{r} \right).
\end{equation}
Similarly, the wrenches are related according to
\begin{equation}
\vc{w}_o = \mx{G}^T \vc{w}_p.
\end{equation}
\end{prop}
\begin{proof}
By changing the point of reference, we find 
\begin{subequations}
\begin{align}
	\vec{v}_p &= \vec{v}_o + \vec{\omega}_o \times \overrightarrow{OP}, \label{eq:shiftVel} \\
	\vec{\omega}_p &= \vec{\omega}_o.
\end{align}
\end{subequations}
Rewriting~\eqref{eq:shiftVel} in the frame $\{\mathrm{O}\}$, we obtain
\begin{align*}
\mx{R}  \vc{v}_p &=\vc{v}_o + \omega_o \hat{\vc{k}} \times (x_r \hat{\vc{i}} + y_r \hat{\vc{j}} ) \\
	&=\vc{v}_o + (x_r  \hat{\vc{j}} -  y_r \hat{\vc{i}}) \omega_o,
\end{align*}
where $\hat{\vc{i}}$, $\hat{\vc{j}}$, and $\hat{\vc{k}}$ denote unit coordinate vectors.
Similarly for the forces, we have
\begin{subequations}
\begin{align}
	\vec{f}_o &= \vec{f}_p \\
	\vec{m}_o &= \vec{m}_p + \overrightarrow{OP} \times \vec{f}_p. \label{eq:shiftForce}
\end{align}
\end{subequations}
Rewriting~\eqref{eq:shiftForce} in the frame $\{\mathrm{O}\}$ results in
\begin{align*}
 m_o \hat{\vc{k}} & =m_p \hat{\vc{k}} + (x_r \hat{\vc{i}} + y_r \hat{\vc{j}}) \times ( f_{xo} \hat{\vc{i}} + f_{yo} \hat{\vc{j}} ) \\
	&=(m_p  + x_r f_{yo} - y_r f_{xo} ) \hat{\vc{k}} .
\end{align*}
Additionally, the change of the frame from  $\{\mathrm{P}\}$ to  $\{\mathrm{O}\}$ requires
\begin{align*}
\vc{f}_o = \mx{R} \vc{f}_p.
\end{align*}

The proof is completed by rewriting the results in matrix form.
\end{proof}

Using the Coulomb model of friction between surfaces, the friction wrench with respect to point $o$ can be calculated  as
\begin{subequations} \label{eq:LSInts}
\begin{align} 
    \vc{f}_o &= - \int_D \dfrac{\vc{v}(\vc{r})}{\norm{\vc{v}(\vc{r})}} \mu_r p(\vc{r}) \dA, \label{eq:fo}\\
    \vc{m}_o &= - \int_D \dfrac{(\vc{r}-\vc{o}) \times \vc{v}(\vc{r})}{\norm{\vc{v}(\vc{r})}} \mu_r p(\vc{r}) \dA, \label{eq:mo}
\end{align}
\end{subequations}
where $p(\vc{r})$ denotes the pressure and $\vc{v}(\vc{r})$ denotes the relative linear velocity between sliding surfaces at position $\vc{r}$. The integral is calculated over the area $D$. 
Based on the assumed quadratic model of the limit surfaces, the relation between~\eqref{eq:fo} and~\eqref{eq:mo} can be approximated by an implicit function
\begin{align}
H(\vc{w}) := \vc{w}^T \mx{A} \vc{w} = 1, \label{eq:LS}
\end{align}
for a positive definite matrix $\mx{A} \in \mathbb{R}^{3\times 3}$.
The corresponding twist is parallel to the gradient of $H(\vc{w})$. Thus, 
\begin{align}
\vc{\nu} &= -k' \nabla H(\vc{w}) \nonumber \\
&= - k  \mx{A} \vc{w}, \quad k \geq 0. \label{eq:gradH}
\end{align}
Note that for a given $\vc{w}$ applied to an object sliding on a surface, there will be no relative motion if
\begin{align*}
H(\vc{w}) < 1,
\end{align*}
and the object will be accelerating if $H(\vc{w})$ is larger than one.
By combining~\eqref{eq:LS} and~\eqref{eq:gradH}, it is possible to eliminate $k$ and hence to find wrenches as a function of the twist
\begin{align}
\vc{w} = - \dfrac{\mx{A}^{-1} \vc{\nu}}{\sqrt{\vc{\nu}^T \mx{A}^{-1} \vc{\nu}}}. \label{eq:kLS}
\end{align}

\begin{prop}
Assume that the limit surface calculated with respect to frame  $\{\mathrm{O}\}$ can be represented by
\begin{equation}
	\vc{w}_o^T \mx{A} \vc{w}_o = 1,
\end{equation}
where $\mx{A}$ is a positive definite matrix.
Then, the limit surface with respect to frame $\{\mathrm{P}\}$, which has the relative coordinates $\left[\vc{r},\, \theta_{r} \right]^T$ is
\begin{equation}
	\vc{w}_p^T \hat{\mx{A}} \vc{w}_p = 1,
\end{equation}
where 
\begin{equation}
    \hat{\mx{A}} = \mx{G} \mx{A} \mx{G}^T
\end{equation}
is a positive definite matrix and $\mx{G} = \mx{R}^T(\theta_r)\mx{J}\left( \vc{r} \right)$.
\end{prop}
\begin{proof}
The result is achieved by the direct application of Proposition~\ref{thm:trans}. For the positive definiteness, note that 
\begin{align*}
    \vc{w}^T \hat{\mx{A}} \vc{w} = (\mx{G}^T \vc{w})^T \mx{A} (\mx{G}^T \vc{w}) \geq 0.
\end{align*}
Since $\mx{G}$ is full rank, $\mx{G}^T \vc{w}$ is zero if and only if $\vc{w} = \vc{0}$. Consequently, the matrix $\hat{\mx{A}}$ is positive definite.
\end{proof}

The following theorem shows that a limit surface characterized by any positive definite matrix can be assumed as a diagonal matrix with respect to a frame assigned at the COP.
\begin{thm}\label{thm:decomp}
	Any positive definite matrix $\mx{A} \in \mathbb{R}^{3\times 3}$ can be decomposed as 
\begin{align}
	\mx{A} = \mx{R}^T \mx{J} \mx{\Lambda} \mx{J}^T \mx{R}, \label{eq:decomp}
\end{align}
where $\mx{\Lambda}$ is a diagonal matrix, and $\mx{R}$ and $\mx{J}$ are rotation and Jacobian matrices as defined in~\eqref{eq:rotz} and~\eqref{eq:jac}, respectively.
\end{thm}
\begin{proof}
	See appendix~\ref{sec:appx2}
\end{proof}

\subsection{Force and velocity relations}

The limit surfaces and the relation between the friction wrench exerted on the hand through the patch $\mx{w}_h $ and the wrench affecting the object $\mx{w}_o $ are:
\begin{align}
\vc{w}_o^T \mx{A} \vc{w}_o &= 1, \label{eq:HWo}\\
\vc{w}_h^T \mx{B} \vc{w}_h &= 1, \label{eq:HWh} \\ 
\vc{w}_o - \mx{G}^T \vc{w}_h &= \vc{0}, \label{eq:forceBalance}
\end{align}
where $\mx{G} := \mx{G}(\vc{q}_{rel})$ denotes the Jacobian corresponding to the relative coordinates of frame $\{\mathrm{H}\}$ with respect to $\{\mathrm{O}\}$.
Equation~\eqref{eq:forceBalance} is derived from the fact that the wrenches on the object sum to zero under the assumption of quasi-static manipulation, i.e., the inertial forces are negligible.
Additionally, we have these velocity relations
\begin{align}
\vc{\nu}_o &= - k_1  \mx{A} \vc{w}_o, \quad k_1 \geq 0 \label{eq:oVelForce}\\
\vc{\nu}_{rel} &= - k_2 \mx{B} \vc{w}_h,\quad k_2 \geq 0 \label{eq:hVelForce}\\
\vc{\nu}_{rel} &=  \vc{\nu}_h - \mx{G} \vc{\nu}_o, \label{eq:Vrel}
\end{align}
where $\vc{\nu}_{rel}$ denotes the relative twist of the patch with respect to the object expressed in $\{\mathrm{H}\}$. 
Equations~\eqref{eq:oVelForce} and~\eqref{eq:hVelForce} are the counterparts of~\eqref{eq:gradH} while~\eqref{eq:Vrel} is obtained by first transforming $\vc{\nu}_o$ to the frame of the patch and then subtracting it from the twist of the patch.

\subsection{Solution}
Using~\eqref{eq:forceBalance} it is possible to rewrite~\eqref{eq:HWo} as
\begin{align}
\vc{w}_h^T \hat{\mx{A}} \vc{w}_h = 1, \label{eq:HWot}
\end{align}
where $\hat{\mx{A}} = \mx{G} \mx{A} \mx{G}^T$ characterizes the limit surface of the object at frame $\{\mathrm{H}\}$. 
By solving the generalized eigenvalue problem $\mx{B} \mx{\Phi} = \hat{\mx{A}} \mx{\Phi} \mx{\Lambda}$, we can simultaneously diagonalize $\hat{\mx{A}}$ and $\mx{B}$ such that 
\begin{align*}
\mx{\Lambda} &= \mx{\Phi}^T \mx{B} \mx{\Phi}, \\
\mx{I} &= \mx{\Phi}^T \hat{\mx{A}} \mx{\Phi},
\end{align*}
where $\mx{I} \in \mathbb{R}^{3\times 3}$ denotes the identity matrix.
Thus, by applying $\vc{w}_h = \mx{\Phi} \vc{w}$ we transform~\eqref{eq:HWh} and~\eqref{eq:HWot} to 
\begin{subequations} \label{eq:intersecES}
\begin{align}
\vc{w}^T \mx{\Lambda} \vc{w} &= 1, \label{eq:ellips}\\
\vc{w}^T \vc{w} &= 1. \label{eq:sphere}
\end{align}
\end{subequations}
Moreover, by subtracting~\eqref{eq:sphere} from \eqref{eq:ellips}, we find the normal form
\begin{subequations} \label{eq:intersec}
\begin{align}
\vc{w}^T \mx{C} \vc{w} &= 0, \label{eq:intersec1}\\
\vc{w}^T \vc{w} &= 1,  \label{eq:intersec2}
\end{align}
\end{subequations}
where $\mx{C} := \mx{\Lambda} - \mx{I}$ is a diagonal matrix.
Note that if there is a solution to~\eqref{eq:intersec}, it is possible to recover the wrenches at the patch and the object frames using the following relations
\begin{subequations} \label{eq:transform}
\begin{align}
	\vc{w}_h &= \mx{\Phi} \vc{w}, \label{eq:Wh}\\
	\vc{w}_o &= \mx{G}^T \vc{w}_h
\end{align}
\end{subequations}

In view of~\eqref{eq:intersecES}, feasible wrenches $\vc{w}$ lie on the intersection of an ellipsoid with the unit sphere.
Accordingly, there are several possible cases:
\begin{itemize}
	\item The limit surface of the object lies entirely inside the limit surface of the patch, hence $\mx{C} \prec 0$. Since any required forces for sliding can be provided through the patch, the hand sticks to the object ($\vc{\nu}_{rel} = \vc{0}$). The only possible mode in this case is called \emph{sticking}.
	\item The limit surface of the patch is entirely contained in the limit surface of the object, hence $\mx{C} \succ 0$. In this case, the hand cannot provide enough force through the patch for sliding the object against the surface, hence the object remains still and the patch slides against it ($\vc{\nu}_{o} = \vc{0}$). We call the corresponding mode \emph{slipping}.
	\item Otherwise, there exists a $\vc{\nu}_h$ for which the hand can move the object while allowing it to pivot. We call this mode \emph{pivoting}.  An example in which pivoting is possible is illustrated in Figure~\ref{fig:interSphere}.
\end{itemize}

\begin{figure}
	\begin{center}
		\includegraphics[width=0.8\linewidth]{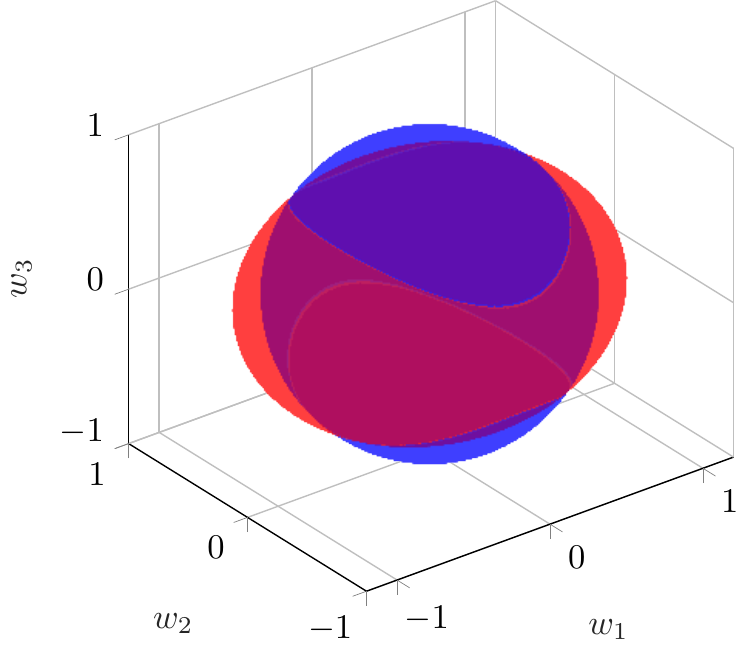}
		\caption{Visualization of Eq.~\eqref{eq:intersecES} for an example where pivoting mode is possible. The vector $\vc{w}$ is unitless.}
		\label{fig:interSphere}
	\end{center}
\end{figure}

Using the transformations~\eqref{eq:transform}, it is also possible to rewrite ~\eqref{eq:oVelForce}--\eqref{eq:Vrel} to obtain
\begin{align}
	\mx{\Phi}^T \mx{G} \vc{\nu}_o &= - k_1  \vc{w}, \quad k_1 \geq 0 \label{eq:oVelForceT} \\
	\mx{\Phi}^T \vc{\nu}_{rel} &= - k_2 \mx{\Lambda} \vc{w},\quad k_2 \geq 0 \label{eq:rVelForceT}\\
	\vc{\nu}_{rel} &=  \vc{\nu}_h - \mx{G} \vc{\nu}_o. \label{eq:VrelT}
\end{align}
From~\eqref{eq:oVelForceT}--\eqref{eq:VrelT}, we conclude
\begin{align*}
\tilde{\vc{\nu}}_h  = - (k_1 \mx{I} + k_2 \mx{\Lambda}) \vc{w},
\end{align*}
where $\tilde{\vc{\nu}}_h := \mx{\Phi}^T \vc{\nu}_h$.
Let us define $\alpha = \dfrac{k_2}{k_1} \geq 0$. Accordingly, 
\begin{align}
	\vc{w} = -\dfrac{1}{k_1} (\mx{I} + \alpha \mx{\Lambda})^{-1} \tilde{\vc{\nu}}_h. \label{eq:wvh}
\end{align}
Substituting~\eqref{eq:wvh} into~\eqref{eq:intersec1} results in
\begin{align}
    \tilde{\vc{\nu}}_h^T \mx{C} (\mx{I} + \alpha \mx{\Lambda})^{-2} \tilde{\vc{\nu}}_h = 0, \label{eq:alpha4vh}
\end{align}
which is equivalent to
\begin{multline}
	    c_1 \left(\dfrac{\tilde{v}_{xh}}{\alpha \lambda_1 + 1} \right)^2 +
    c_2 \left(\dfrac{\tilde{v}_{yh}}{\alpha \lambda_2 + 1} \right)^2 +
    c_3 \left(\dfrac{\tilde{\omega}_{h}}{\alpha \lambda_3 + 1} \right)^2 \\= 0,  \label{eq:solAlpha}
\end{multline}
where $c_i$ and $\lambda_i$ are the diagonal elements of $\mx{C}$ and $\mx{\Lambda}$, respectively.
Equation~\eqref{eq:solAlpha} can be solved for $\alpha$. Afterwards, by substituting~\eqref{eq:wvh} into~\eqref{eq:intersec2}, it is possible to calculate $k_1$.  

A relation between $\vc{\nu}_o$ and $\vc{\nu}_h$ can also be found by substituting~\eqref{eq:wvh} back to~\eqref{eq:oVelForceT}
\begin{align*}
{_h}\tilde{\vc{\nu}}_o =  (\mx{I} + \alpha \mx{\Lambda})^{-1} \tilde{\vc{\nu}}_h.
\end{align*}
After some algebraic manipulations, we have
\begin{align}
{_h}\vc{\nu}_o &= \hat{\mx{A}} (\hat{\mx{A}} + \alpha \mx{B})^{-1} \vc{\nu}_h \nonumber \\
	&= (\mx{I} + \alpha \mx{B} \hat{\mx{A}}^{-1})^{-1} \vc{\nu}_h, \label{eq:hvovh}
\end{align}
where ${_h}\vc{\nu}_o := \mx{G} \vc{\nu}_o$ is the twist of the object expressed in $\{\mathrm{H}\}$.
Using~\eqref{eq:VrelT}, we find the relative twist to be
\begin{align}
	\vc{\nu}_{rel} &= \left(\mx{I} + (\alpha \mx{B} \hat{\mx{A}}^{-1})^{-1} \right)^{-1} \vc{\nu}_h \nonumber \\
		&= \alpha (\alpha \mx{I} + \hat{\mx{A}}\mx{B}^{-1})^{-1} \vc{\nu}_h. \label{eq:vrvh}
\end{align}

When the patch slides against the object, there is a pivot point, which can be determined by finding the point where the object and the patch have the same velocity. In other words, the \emph{pivot point} is the instantaneous center of rotation (COR) between the patch and the object. Using the velocity transfer relation according to Proposition~\ref{thm:trans}, we conclude the location of the pivot point in the hand frame is
\begin{align}
	\vc{p} := [x_p,\, y_p]^T =\dfrac{1}{ \omega_r} [-v_{yr},\, v_{xr}]^T, \label{eq:ppoint}
\end{align}
where $\vc{\nu}_{rel} = \left[v_{xr},\, v_{yr},\, \omega_r \right]^T $ denotes the relative twist of the patch with respect to the object expressed in $\{\mathrm{H}\}$.

In sticking mode, the pivot point is indeterminate and we may choose any point, e.g., the origin of $\{\mathrm{H}\}$. However, at the boundary of pivoting and sticking modes, it is possible to make the pivot point a continuous function by evaluating the limit as $\alpha \to 0$. In view of~\eqref{eq:vrvh}, this is equivalent of substituting~$\vc{\nu}_{rel}$ in~\eqref{eq:ppoint} with
\begin{align*}
	\bar{\vc{\nu}}_{rel} = \mx{B}\hat{\mx{A}}^{-1} \vc{\nu}_h.
\end{align*}

\subsection{Regions of validity} \label{sec:regVal}
If there is $\alpha > 0$ to satisfy~\eqref{eq:solAlpha}, the pivoting mode is active, which implies having a finite pivot point. Otherwise, the wrenches can be calculated to identify which mode is valid. In sticking mode, from the twist of the patch and the fact that the object slides on the surface, we can easily calculate $\vc{w}$
\begin{subequations}
\begin{align}
\tilde{\vc{\nu}}_h &= -k_1 \vc{w}, \label{eq1:objSlide} \\
1 &= \vc{w}^T \vc{w}. \label{eq2:objSlide}
\end{align}
\end{subequations}
Then, the sticking mode is valid if the contact between the patch and the object can be sustained by the friction, i.e.,
\begin{align}
\vc{w}^T \mx{\Lambda}\vc{w}  < 1. \label{eq:condStickW}
\end{align}
Subtracting~\eqref{eq2:objSlide} from~\eqref{eq:condStickW} results in
\begin{align}
\vc{w}^T \mx{C} \vc{w} < 0. \label{eq:stickCond}
\end{align}
Using~\eqref{eq1:objSlide}, it is possible rewrite the condition as
\begin{align}
\tilde{\vc{\nu}}_h^T \mx{C} \tilde{\vc{\nu}}_h < 0. \label{eq:cond0Stick}
\end{align}

Note that whenever $\alpha = 0$, the relative velocity is zero and hence the mode is sticking. Since in this case Equation~\eqref{eq:alpha4vh} degenerates to condition~\eqref{eq:cond0Stick} with an equality sign, we extend the condition to include also its boundary. Accordingly, in sticking mode
\begin{align}
\tilde{\vc{\nu}}_h^T \mx{C} \tilde{\vc{\nu}}_h \leq 0, \label{eq:condStick}
\end{align}
or equivalently
\begin{align}
	\vc{\nu}_h^T \hat{\mx{A}}^{-1} \left( \mx{B} - \hat{\mx{A}}  \right)  \hat{\mx{A}}^{-1}\vc{\nu}_h
	 \leq 0. \label{eq:motionCone}
\end{align}
Similarly, in slipping mode
\begin{subequations}
\begin{align}
	\tilde{\vc{\nu}}_h &= -k_2 \mx{\Lambda} \vc{w}, \label{eq1:handSlide} \\
	1 &=\vc{w}^T \mx{\Lambda}\vc{w}. \label{eq2:handSlide} 
\end{align}
\end{subequations}And the mode is valid if 
\begin{align}
	\vc{w}^T \vc{w} < 1. \label{eq:condSlipW}
\end{align}
Subtracting~\eqref{eq:condSlipW} from~\eqref{eq2:handSlide} results in
\begin{align*}
	\vc{w}^T \mx{C} \vc{w} > 0. 
\end{align*}
Using a similar argument as mentioned before, we extend the boundary to include the case $\alpha \to \infty$ and express the condition using~\eqref{eq1:handSlide} as
\begin{align}
	\tilde{\vc{\nu}}_h^T \mx{C}  \mx{\Lambda}^{-2}  \tilde{\vc{\nu}}_h = \tilde{\vc{\nu}}_h^T \mx{C}  (\mx{C} +\mx{I})^{-2}  \tilde{\vc{\nu}}_h \geq 0, \label{eq:condSlip}
\end{align}
or equivalently
\begin{align}
	\vc{\nu}_h^T  \mx{B}^{-1} \left( \hat{\mx{A}} - \mx{B}  \right)  \mx{B}^{-1} \vc{\nu}_h \geq 0.
\end{align}

%% file: 03_dynSys.tex
The quasi-static behavior of the system is determined based on  the modes of the system. There are three modes depending on which contact surface can provide enough friction to avoid slippage. To determine the active mode, i.e., the discrete state of the system, when $\vc{\nu}_h \neq 0$ first we solve~\eqref{eq:solAlpha}. If there is a positive real solution, then the pivoting mode is selected. Otherwise, depending on whether the condition~\eqref{eq:condStick} or~\eqref{eq:condSlip} is fulfilled, sticking or slipping mode is selected, respectively. When $\vc{\nu}_h = 0$, the definition of the mode is somewhat arbitrary. Hence, if neither $\mx{C} \prec 0$ nor $\mx{C} \succ 0$, we set the mode to pivoting. This procedure is summarized in Algorithm~\ref{alg:mode}.

\begin{algorithm}[b]
	\caption{Mode selection}\label{alg:mode}
	\begin{minipage}[t]{.57\columnwidth}
	$\vc{\nu}_h \neq 0$
		\begin{algorithmic}[1]
		\State Solve for $\alpha$ in \eqref{eq:solAlpha}
		\If{ $\exists \alpha > 0$ }
			\State mode $\gets$ pivoting
		\ElsIf { $\tilde{\vc{\nu}}_h^T \mx{C} \tilde{\vc{\nu}}_h \leq 0$ }
			\State mode $\gets$ sticking
		\Else
			\State mode $\gets$ slipping
		\EndIf
		\end{algorithmic}
	\end{minipage}%
	\begin{minipage}[t]{.43\columnwidth}
	$\vc{\nu}_h = 0$
	\begin{algorithmic}[1]
		\If{ $\mx{C} \prec 0$ }
			\State mode $\gets$ sticking
		\ElsIf { $\mx{C} \succ 0$ }
			\State mode $\gets$ slipping
		\Else
			\State mode $\gets$ pivoting
		\EndIf
	\end{algorithmic}
	\end{minipage}
\end{algorithm}

By integrating the twists in the global frame, we find the generalized coordinates of the patch and the object. This implies the dynamical system
\begin{subequations} \label{eq:dynSys1}
\begin{align}
	\dot{\vc{q}}_h &= \mx{R}(\theta_h) \vc{\nu}_h, \\
	\dot{\vc{q}}_o &= \mx{R}(\theta_o) \vc{\nu}_o. \label{eq:dynO}
\end{align}
\end{subequations}

In pivoting mode, i.e., when $\exists \alpha > 0$, using~\eqref{eq:hvovh} we find
\begin{align}
\vc{\nu}_o  =  \mx{G}^{-1} \left( \mx{I} + \alpha \mx{B} \hat{\mx{A}}^{-1}\right)^{-1} \vc{\nu}_h,\label{eq:voPivot}
\end{align}
where $\mx{G} := \mx{G}(\vc{q}_{rel})$ and $\vc{q}_{rel} = \mx{R}(-\theta_o) \left(\vc{q}_{h} - \vc{q}_{o}\right)$.
Otherwise, in sticking mode, using~\eqref{eq:VrelT} and the fact that $\vc{\nu}_{rel} = 0$ we obtain
\begin{align}
	\vc{\nu}_o = \mx{G}^{-1} \vc{\nu}_h. \label{eq:voStick}
\end{align}
In slipping mode, the velocity of the object is zero,
\begin{align}
	\vc{\nu}_o = \vc{0}.
\end{align}

The wrenches at the hand $\vc{w}_h$ can also be expressed in terms of the states and $\vc{\nu}_o$, which in turn is a function of the input. Using~\eqref{eq:kLS} together with~\eqref{eq:forceBalance}, we conclude
\begin{align}
	\vc{w}_h = - \dfrac{ \mx{G}^{-T} \mx{A}^{-1} \vc{\nu}_o}{\sqrt{\vc{\nu}_o^T \mx{A}^{-1} \vc{\nu}_o}},
\end{align}
Accordingly, in pivoting mode using~\eqref{eq:voPivot} we have
\begin{align}
	\vc{w}_h = - \dfrac{1}{k_1} \left( \hat{\mx{A}} + \alpha \mx{B}\right)^{-1} \vc{\nu}_h,
\end{align}
where
\begin{align*}
	k_1 = \left({\vc{\nu}_h^T \left( \hat{\mx{A}} + \alpha \mx{B}\right)^{-1} \hat{\mx{A}} \left( \hat{\mx{A}} + \alpha \mx{B}\right)^{-1} \vc{\nu}_h}\right)^{1/2}.
\end{align*}
And in sticking mode,
\begin{align}
	\vc{w}_h = - \dfrac{ \hat{\mx{A}}^{-1} \vc{\nu}_h}{\sqrt{\vc{\nu}_h^T \hat{\mx{A}}^{-1} \vc{\nu}_h}}.
\end{align}
In slipping mode, $\vc{\nu}_r = \vc{\nu}_h$. Therefore, using~\eqref{eq:kLS} it is concluded that
\begin{align}
	\vc{w}_h = - \dfrac{\mx{B}^{-1} \vc{\nu}_h}{\sqrt{\vc{\nu}_h^T \mx{B}^{-1} \vc{\nu}_h}}.
\end{align}

\subsection{Effect of normal forces}\label{sec:normalforce}
Our formulation is generic with respect to $\mx{A}$ and $\mx{B}$ describing the limit surfaces, as long as the matrices are positive definite. In fact, both matrices can be time-varying, specifically when the COP of the object does not have a fixed transformation to its COM or when the patch is deforming as a result of variations in normal or tangential forces.

For surfaces with homogeneous friction coefficients and symmetrical pressure distributions, with no deformation of contact areas as a result of varying normal forces, the trajectory depends only on the ratio between normal forces at the Hand-Object (HO) and the Object-Environment (OE) contacts, and not their absolute values. The reason is that given these assumptions, the normal forces as well as the friction coefficients can be factorized from $\mx{A}$ and $\mx{B}$, and in the solution only the ratio will appear. Nevertheless, the friction forces will be scaled.

In general, when the normal force at the patch is changed, the frictions related to HO and OE are not proportionally changed. Firstly, the lower surface has to additionally support the weight of the object, secondly the pressure distribution may vary and become stronger closer to the patch, and thirdly deformation of the patch may increase its contact area.

To exactly model the effect of normal force, it is required to know the pressure distributions and their variation. This is not a simple task as the pressure distribution depends in general on the stiffness of the contact surfaces, geometry of the contact, and relative velocities. Particularly, the friction patch may go through large deformations as a function of normal forces.

To get an understanding of the effect of normal force, consider a special case where a flat object and a sphere-shaped soft finger following a Hertzian law are in contact. Denoting the normal force on the sphere by $f_n$, the pressure distribution at radius $r$ is~\citep{johnson_1985} 
\begin{align}
	p(r) = p_0\left(1-\dfrac{r^2}{a^2}\right)^{1/2},
\end{align}
where
\begin{align}
	p_0 = \dfrac{3}{2\pi a^2} f_n
\end{align}
and $a$ is the radius of the contact area.
Using this pressure distribution, Equation~\eqref{eq:LSInts} allows us to calculate the maximum friction force and torque
\begin{align}
	f_{max} &= \mu f_n, \\
	m_{max} &= \mu \dfrac{3 \pi}{16}  a f_n.
\end{align}
By changing the normal force, the radius of the contact area increases according to
\begin{align}
	a = \left(\dfrac{3}{4} \dfrac{R}{E^*} f_n\right)^{1/3},
\end{align}
where $R$ is the radius of the sphere and $E^*$ is the effective elastic modulus.
As it can be seen, while tangential forces depend linearly on the normal force, the torque has a nonlinear dependence because of the increase in contact area. Accordingly, it is possible to change the ratio of the torsional to tangential friction of the patch.

Another observation is that by pressing the patch harder, the COP of the object shifts more toward the patch. Although modeling the exact physical phenomenon is complicated, we can easily incorporate this effect using a computational model. For example, define $s \in \mathbb{R}$ to be a value between zero and one characterizing the percentage of the shift of the COP
\begin{align}
	s = 1 - (c \dfrac{f_n}{m g} + 1)^{-\delta}, \label{eq:COPmodel}
\end{align}
where $c$ and $\delta$ are model parameters and $m g$ is the weight of the object.
Then, if the limit surface at the COP of the object is characterized by $\mx{A}_{COP}$ and~$\vc{r}$ denotes the relative position of the hand frame w.r.t. the object frame, the limit surface at the object frame is
\begin{align}
	\mx{A} =  \mx{J}(-s \vc{r}) \mx{A}_{COP} \mx{J}^T(-s \vc{r}).
\end{align}
A similar approach can be used to compensate for the shift of COPs due to relative velocities.
See Appendix~\ref{sec:appx3} for experimental validation of the proposed model for the shift of COP.

%% file: 04_solProp.tex
In this section, we describe a number of properties of the quasi-static sliding motion. Some are general properties concerning the well posedness of the problem and some are useful results for planning and control.

If none of the diagonal elements of $\mx{C}$ are zero, i.e, the matrix is full rank, we can prove that there is always a unique solution to the system. First we prove that there could exist at most one solution in pivoting mode. Secondly, we prove that there exists at least one active mode and it is impossible to have any two modes active at the same time. Note that when $\mx{C}$ is rank deficient, $\vc{\nu}_o$ is indeterminate for certain $\vc{\nu}_h$ since the forces can be balanced for any value of $\alpha \geq 0$. For example, in one-dimensional space, as a consequence of Coulomb friction model, if the friction force at HO contact is exactly the same as the friction at OE contact, the object can have any velocity in the range from zero to the hand velocity.

\begin{thm}\label{thm:no2alpha}
If $\mx{C}$ is full rank and $\vc{\nu}_h \neq 0$, there is at most one positive solution ($\alpha > 0$) to Equation~\eqref{eq:solAlpha}.
\end{thm}
\begin{proof}
See Appendix~\ref{sec:appx0}.
\end{proof}

\begin{thm} \label{thm:unique}
If $\mx{C}$ is full rank and $\vc{\nu}_h \neq \vc{0}$, the mode is uniquely determined. 
\end{thm}
\begin{proof}
See Appendix~\ref{sec:appx}.
\end{proof}
Note that the transition between modes can happen smoothly. For example, the pivot point can gradually shift to infinity as the patch starts to slide against the object or go to a limit value as the patch and the object approach the same velocity.

\begin{thm}\label{thm:vScale}
Scaling the twist of the patch by a constant results in the twist of the object being scaled by the same constant, i.e.,
if $\vc{\nu}_h$ results in  $\vc{\nu}_o$, then $c \vc{\nu}_h$ results in  $ c \vc{\nu}_o$.
\end{thm}
\begin{proof}
Note that $\alpha$ which is the solution to~\eqref{eq:solAlpha} is not affected by scaling $\vc{\nu}_h$. The proof is completed by considering the relations~\eqref{eq:voPivot} and~\eqref{eq:voStick}.
\end{proof}
\begin{cor} \label{cor:revTwist}
Reversing the twist of the patch reverses the twist of the object, which is proved by setting $c = -1$.
\end{cor}
\begin{cor} \label{cor:perTwist}
When $\omega_h = 0$, the location of the pivot point as a function of the velocity direction is periodic with period $\pi$. The proof is straightforward by considering~\eqref{eq:vrvh} and~\eqref{eq:ppoint}.
\end{cor}
Note that Corollary~\ref{cor:revTwist} does not guarantee that reversing a trajectory will bring the object back to its initial state. This is due to the sensitivity of the dynamical system to perturbations, e.g., reversing from the vicinity of a fixed-point is very sensitive to noise.

\begin{thm}
Given the quasi-static assumption for sliding, the path of the object is invariant under scaling of the patch twist according to $c \vc{\nu}_h(ct)$ for any constant $c \in \mathbb{R}  \neq 0$. 
\end{thm}
\begin{proof}
Let us define $\tilde{\vc{q}}_o(t) := \vc{q}_o(c t)$ where $\vc{q}_o(t)$ is the solution of~\eqref{eq:dynO}. Accordingly,  we conclude
\begin{align}
	\dot{\tilde{\vc{q}}}_o = \mx{R}(\tilde{\theta}_o) c\vc{\nu}_o(c t).	
\end{align}

From Theorem~\ref{thm:vScale}, it is clear that $c \vc{\nu}_h( c t)$ results in $c \vc{\nu}_o(c t)$. Thus, $\tilde{\vc{q}}_o(t)$ is the solution of the system with the scaled input velocity. Since a path is independent of its parametrization, the path defined by $\vc{q}_o(t)$ is the same as $\tilde{\vc{q}}_o(t)$.
\end{proof}
\begin{cor}
Starting from the same initial configuration and moving the patch along a line at a constant velocity, the distance required for the patch to move for causing a certain change in the orientation of the object, $\Delta \theta_o$, is independent of the velocity.
\end{cor}

It is possible to find the set of patch twists corresponding to each mode of the system. As we will prove, these sets define cones in twist space. This allows defining a similar concept to \emph{motion cones} appeared in~\citep{mason1986mechanics,lynch1996stable}, defining the set of possible velocities for stable pushing, i.e., when the pusher does not slide against the object. In this article, we use the term more generically to refer to twist subspaces corresponding to any modes, i.e., we may use slipping, sticking, and pivoting motion cones. However, unless it is specified, by a motion cone we also mean the subspace corresponding to sticking mode.  

For constant twists of the friction patch, fixed-points of the system are in sticking mode, i.e., when the patch configuration does not change with respect to the object. This is analogous to stable pushing.  For straight line motions, starting from the pivoting mode, a possible fixed-point is reached at the limit when $\omega_o = 0$.

\begin{thm} \label{thm:coneReg}
The boundary of each mode in twist space of the patch is a conical surface. The set of patch motions for which the object sticks to the patch or slips against it defines a cone. Accordingly, the pivoting mode is the intersection of the complement of these two cones.
\end{thm}
\begin{proof}
The theorem is a direct result of  the discussions of the regions of validity in Section~\ref{sec:regVal}.
According to~\eqref{eq:condStick}, sticking mode is valid if
\begin{align}
\tilde{\vc{\nu}}_h^T \mx{C} \tilde{\vc{\nu}}_h \leq 0
\end{align}
and the boundary is obtained with the equality sign.
Similarly, for slipping mode from~\eqref{eq:condSlip} we have
\begin{align}
\tilde{\vc{\nu}}_h^T \mx{C} (\mx{C} + \mx{I})^{-2} \tilde{\vc{\nu}}_h \geq 0.
\end{align}
These equations define closed cone sets in the form of the interior and the exterior of elliptic cones, depending on the signs of the diagonal elements of $\mx{C}$. The intersection of the complement of these sets defines the region of validity for pivoting motion, i.e., 
\begin{align}
\left\{ \tilde{\vc{\nu}}_h :
\tilde{\vc{\nu}}_h^T \mx{C} \tilde{\vc{\nu}}_h > 0 \wedge
\tilde{\vc{\nu}}_h^T \mx{C} (\mx{C} + \mx{I})^{-2} \tilde{\vc{\nu}}_h < 0 \right\}.
\end{align}
By changing the basis according to $\tilde{\vc{\nu}}_h = \mx{\Phi}^T \vc{\nu}_h$, the proof is completed.
\end{proof}

The following Theorem allows us to define bounds for the direction of the patch velocity that leads to pivoting mode.
\begin{thm}
If $\omega_h=0$, there is either no $\vc{v}_h$ for pivoting mode, or $\vc{v}_h$ lies inside a planar cone or two planar cones.
\end{thm}
\begin{proof}
According to Theorem~\ref{thm:coneReg}, the region of pivoting mode in twist space is the intersection of two cone sets with vertices at the origin. Considering the intersection of this set with the plane $\omega_h=0$ the proof is completed.
\end{proof}
Since the motion cones depend on the position of a patch, we always draw them with respect to $\{\mathrm{H}\}$. Note that the motion cones are three dimensional objects. In two-dimensional space, we draw the intersection of the cone with the $\omega_h$ plane. When $\omega_h=0$ the resulting intersection is a planar cone, but when $\omega_h \neq 0$, the intersection can be any conic section. This implies that changing the magnitude of the linear velocity may result in a mode change.

\begin{figure}
	\begin{center}
		\includegraphics[width=\linewidth]{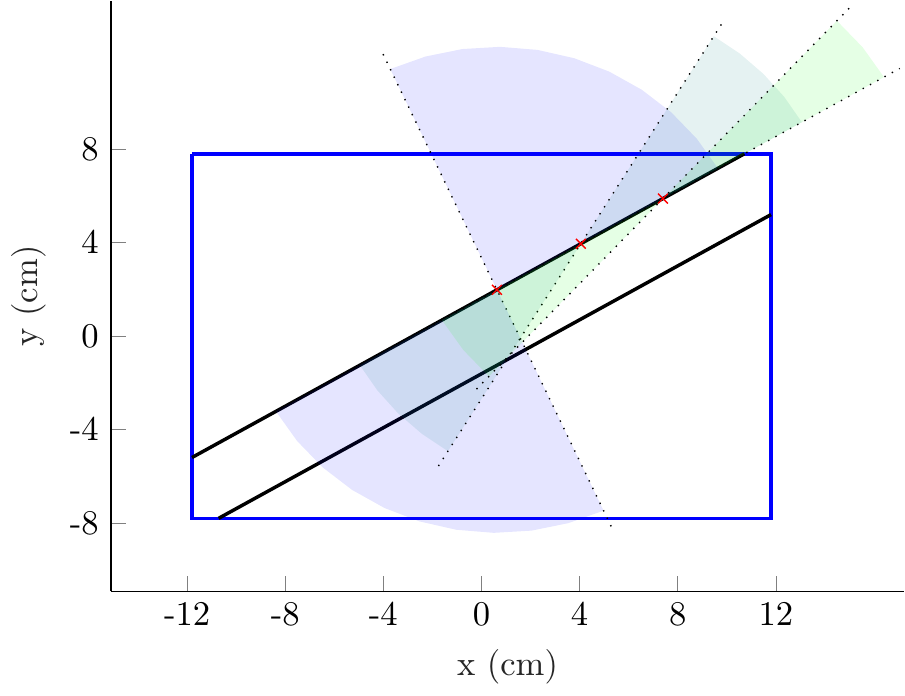}
		\caption{Locus of the patch with $\vc{v}_h = [\sqrt{3},\, 1]^T$ characterizing the boundary of the planar motion cone. The two black lines visualize the locus. For three different hand positions on one of the lines, the sticking motion cones are shaded.}
		\label{fig:hLocus}
	\end{center}
\end{figure}

\begin{thm}\label{thm:hLocus}
Assume variations of $\mx{A}$ with respect to hand placements are negligible. Then, the locus of patch positions with the same $\vc{v}_h$ characterizing the boundary of the planar motion cone ($\omega_h=0$) is composed of at most two parallel lines.
\end{thm}
\begin{proof}
The boundary of the planar motion cone is characterized by $\vc{\nu}_h := \left[\vc{v}^T_h,\, 0 \right]^T$ that solves~\eqref{eq:motionCone} with the equality sign. 
Expanding the expression, we find
\begin{align}
\vc{\nu}_h^T \hat{\mx{A}}^{-1} \mx{B} \hat{\mx{A}}^{-1}\vc{\nu}_h = \vc{\nu}_h^T \hat{\mx{A}}^{-1} \vc{\nu}_h. \label{eq:MCexp}
\end{align}

The Jacobian matrix  $\mx{G}(\vc{q}_{rel})$ is a function of the relative position of the patch and the object $\vc{r} = \left[x_r,\, y_r \right]$. As a result
\begin{align}
	\hat{\mx{A}}^{-1} =  \mx{G}^{-T} \mx{A}^{-1} \mx{G}^{-1}
\end{align}
varies with the change of the patch position.
Since $\omega_h$ is zero, 
\begin{align}
\mx{G}^{-1} \vc{\nu}_h = \mx{R} \vc{\nu}_h.
\end{align}
Thus, the right-hand side of~\eqref{eq:MCexp} is independent of $\vc{r}$.
The left-hand side can be written as
\begin{align}
\vc{u}^T \mx{J}^{-1} \mx{R} \mx{B} \mx{R}^T \mx{J}^{-T} \vc{u}, \label{eq:lhsMC}
\end{align}
where $\vc{u} = \mx{A}^{-1} \mx{G}^{-1} \vc{\nu}_h$ is also independent of $\vc{r}$ because of $\omega_h = 0$. Assuming $\vc{u} = \left[u_1, u_2, u_3\right]^T$, we find 
\begin{align}
\vc{u}^T \mx{J}(-\vc{r}) = \left[u_1, u_2, u_1 y_r -u_2 x_r + u_3 \right]. \label{eq:uTJ}
\end{align}
Accordingly,~\eqref{eq:MCexp} defines a second order equation in $u_1 y_r -u_2 x_r + u_3$. This completes the proof that if there are any solutions for $\left[x_r,\, y_r \right]$, they can form at most two lines with the slope $-u_2/u_1$.
\end{proof}
\begin{cor}\label{cor:hLocus}
If $\mx{A}$ characterizes an isotropic friction, moving $\{\mathrm{H}\}$ along one boundary of a motion cone keeps that boundary intact for the new patch position. 
\end{cor}
\begin{proof}
In this case, $\vc{u} = \left[u_1, u_2, 0\right]^T$ and is aligned with the velocity of the patch in the object frame. From~\eqref{eq:uTJ}, we find that if $\left[u_1, u_2\right]^T$ denotes the unit vector for the boundary of the motion cone of a patch at $\left[x_0, y_0\right]^T$, then $\vc{r} = \left[x_0 + k u_1, y_0 + k u_2 \right]^T$ for $k \in \mathbb{R}$ is the locus of the patch positions with $\left[u_1, u_2\right]^T$ as the boundary of their planar motion cones.
\end{proof}
An example to illustrate the result of Corollary~\ref{cor:hLocus} is given in Figure~\ref{fig:hLocus}.

Note that when the pivot point does not shift significantly, an approximation of the maximum amount of rotation can be achieved by calculating the angle between the direction of $\vc{v}_h$ and the stable boundary of the planar motion cone. Knowing the direction of possible displacements of the patch with respect to the object, the result of Theorem~\ref{thm:hLocus} makes it possible to answer whether the approximation is an upper bound or a lower bound.

\begin{thm}\label{thm:ppLocus}
In a given configuration, the locus of pivot points for all hand velocities is a conic section.
\end{thm}
\begin{proof}
The set of possible wrenches $\vc{w}$ for pivoting mode lives on the intersection of a cone and the unit sphere according to~\eqref{eq:intersec}. Transforming this set to $\vc{\nu}_{rel}$ using~\eqref{eq:rVelForceT} results into a new cone for relative twists. Since all $\vc{\nu}_{rel}$ on a generatrix of the cone result in the same pivot point, it is enough to consider the intersection of this cone with the plane $\omega_r=1$ to find $x_p$ and $y_p$. This completes the proof.
\end{proof}
According to Theorem~\ref{thm:ppLocus}, the locus of pivot points partitions the space into three regions, which are possible to be mapped to the three modes of the system. For example, assuming the locus is an ellipsoid, if the hand slips against the object the center of rotation will be inside the ellipsoid, while in pivoting mode the COR is on the ellipsoid.  Figure~\ref{fig:ppLocus} illustrates the result of Theorem~\ref{thm:ppLocus} for an example where normal forces are varied.

\begin{figure}
	\begin{center}
		\includegraphics[width=0.85\linewidth]{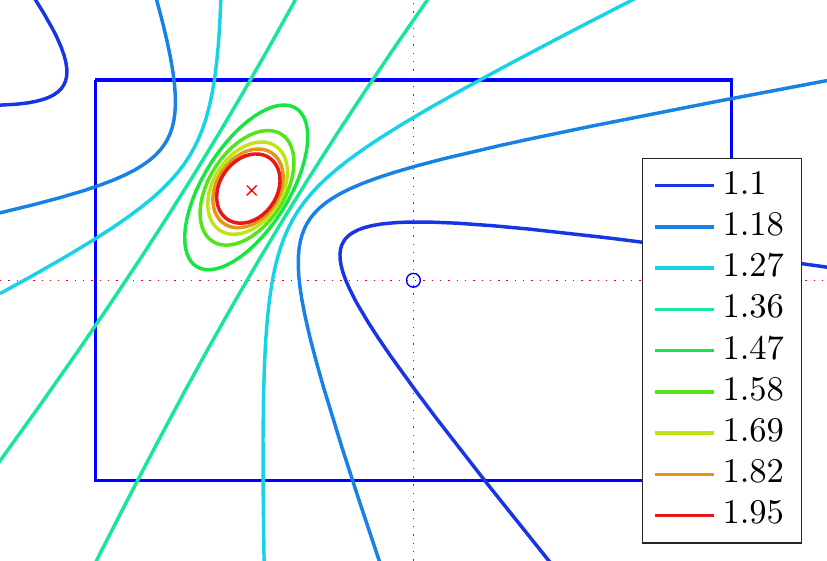}
		\caption{An example demonstrating the locus of possible pivot points for various hand velocities and normal forces. For higher normal forces exerted by the hand, the pivot points get closer to the cross which shows the COP of the friction patch.}
		\label{fig:ppLocus}
	\end{center}
\end{figure}

\begin{thm} \label{thm:stickMaxMargin}
The twists for sticking with the largest margin corresponds to the generalized eigenvector of $\hat{\mx{A}}$ and $\mx{B}$ with the smallest eigenvalue through
\begin{align}
\vc{\nu}_h = k \hat{\mx{A}} \vc{\phi}_{min},\quad k\in \mathbb{R}. \label{eq:vMaxMarigin}
\end{align}
Similarly, the twists corresponding to slipping mode with the largest margin can be found using the eigenvector with the largest eigenvalue
\begin{align}
\vc{\nu}_h = k \hat{\mx{A}} \mx{\Lambda} \vc{\phi}_{max}.
\end{align}
\end{thm}
\begin{proof}
According to~\eqref{eq:stickCond}, sticking with the largest margin is obtained at the minimum of
\begin{align}
c(\vc{w}) = \vc{w}^T \mx{C} \vc{w}. \label{eq:mode}
\end{align}
under the constraint of \eqref{eq2:objSlide}, that is
\begin{align}
\vc{w}^T \vc{w} = 1. \label{eq:constW}
\end{align}
Similarly, slipping with the largest margin is obtained at the maximum of \eqref{eq:mode} under the constraint of~\eqref{eq2:handSlide}, that is
\begin{align}
\vc{w}^T \mx{\Lambda} \vc{w} = 1.  
\end{align}
Given the facts that $\mx{C}$ is diagonal and $\mx{\Lambda} = \mx{C} + \mx{I}$, the optimization problems have trivial solutions with $\vc{w}$ having only one non-zero element in the row corresponding to the smallest or largest element of $\mx{\Lambda}$. Applying the relations~\eqref{eq1:objSlide} and~\eqref{eq1:handSlide} completes the proof.
\end{proof}
The amount of sticking margin can be characterized as the minimum external wrench disturbance at the object frame, $\vc{w}_{dist}$ that can cause a mode change from sticking mode. Certainly, if $\vc{\nu}_h$ is close to an unstable boundary of the motion cone, the effect of the mode change is more drastic.
\begin{thm}
The largest sticking margin is
\begin{align}
	\vc{w}_{dist} =  \sqrt{\abs{\dfrac{\lambda_{min} - 1}{\lambda_{max} - 1}}} \mx{G}^T \vc{\phi}_{max}
\end{align}
where the smallest and largest generalized eigenvalues of $\hat{\mx{A}}$ and $\mx{B}$ are denoted by $\lambda_{min}$ and $\lambda_{max}$, respectively.
\end{thm}
\begin{proof}
From the proof of Theorem~\ref{thm:stickMaxMargin}, we know that the largest margin is achieved for $\vc{w}$ that is zero except at the row corresponding to $\lambda_{min}$. This ensures that $\vc{w}^T\mx{C}\vc{w} < 0$ is minimized given the constraint that $\norm{w} = 1$. To change the mode of the system, the additional wrench has to fulfill
\begin{align}
	(\vc{w}+\Delta\vc{w})^T\mx{C}(\vc{w}+\Delta\vc{w}) = 0.
\end{align}
Since $\mx{C} = \mx{\Lambda} - \mx{I}$ is diagonal, the shortest distance to the cone is simply obtained from
\begin{align}
	(\lambda_{min} - 1) + \Delta w^2 (\lambda_{max} - 1) = 0,
\end{align}
where $\Delta w$ is the only non-zero element of $\Delta \vc{w}$ corresponding to the largest eigenvalue. Accordingly, the solution is
\begin{align}
	\Delta w  = \sqrt{\abs{\dfrac{\lambda_{min} - 1}{\lambda_{max} - 1}}}.
\end{align}
The proof is completed by transforming $\Delta \vc{w}$ to the object frame using~\eqref{eq:transform}.

\end{proof}

When the velocity of the patch is limited to linear velocities ($\omega_h = 0$), possible wrenches generated by the patch lie on a plane. Consequently, the optimum of~\eqref{eq:mode} must be obtained in that subspace.

\begin{thm} \label{thm:linMotionStick}
If the friction is isotropic, for a linear motion, sticking with the largest margin is obtained when the velocity is aligned with the line connecting the COP of the object to the COP of the patch.
\end{thm}
\begin{proof}	
Minimizing~\eqref{eq:mode} under the constraint~\eqref{eq:constW} is equivalent to minimizing 
\begin{align}
	\vc{\nu}_h^T \hat{\mx{A}}^{-1} \mx{B} \hat{\mx{A}}^{-1} \vc{\nu}_h \label{eq:vStickCond}
\end{align}	
under the constraint.
\begin{align}
\vc{\nu}_h^T \hat{\mx{A}}^{-1} \vc{\nu}_h = 1. \label{eq:vObjConst}
\end{align}
Additionally, limiting to linear motions we must impose
\begin{align}
\omega_h = 0.
\end{align}

Since the friction is assumed to be isotropic, we can reorient $\{\mathrm{O}\}$ such that the $y$-axis points to the COP of the patch. Additionally, we align $\{\mathrm{H}\}$ with $\{\mathrm{O}\}$. Now assume $\mx{A} = \diag(\left[a,\, a,\, a_3 \right])$ and $\mx{B} = \diag(\left[b,\, b,\, b_3 \right])$. Accordingly, 
\begin{align}
	\hat{\mx{A}} = \mx{J}\mx{A} \mx{J}^T,
\end{align}
where $\mx{J}:=\mx{J}^T([0\, ,r]^T)$ and $r$ is the distance between the COP of the object and the patch. Since $\omega_h = 0$, for evaluating~\eqref{eq:vStickCond} and~\eqref{eq:vObjConst} only the first two rows and columns of $\hat{\mx{A}}^{-1}$ and  $\hat{\mx{A}}^{-1} \mx{B} \hat{\mx{A}}^{-1}$ are important. These submatrices are equal to 
\begin{align}
	(\hat{\mx{A}}^{-1})_{2 \times 2} &= \dfrac{1}{a} \mx{I}, \\
	(\hat{\mx{A}}^{-1} \mx{B} \hat{\mx{A}}^{-1})_{2 \times 2} &= \dfrac{b}{a^2}
		\begin{bmatrix}
			1 + \dfrac{b_3}{b} r^2	& 0  \\
			0					& 1
	 \end{bmatrix}. \label{eq:mxAiBAi}
\end{align}
The minimum of~\eqref{eq:vStickCond} is obtained along the longest axis of the ellipse defined by~\eqref{eq:mxAiBAi}, that is the $y$-axis.
\end{proof}
Note that the result of Theorem~\ref{thm:linMotionStick} applies to the case when no angular velocity for the patch is permitted. In fact, if it is allowed, under the same assumption of isotropic frictions, the sticking direction with the largest margin will be perpendicular to the line passing through the COPs of the object and the patch. This can be confirmed by evaluating~\eqref{eq:vMaxMarigin} or by evaluating the  generalized eigenvalues of $\hat{\mx{A}}^{-1}$ and $\mx{B}^{-1}$, which gives the solution to~\eqref{eq:vStickCond}--\eqref{eq:vObjConst}.

All the velocities in the motion cone corresponding to sticking mode are fixed-points of the system. However, half of the fixed-points lying on the surface of the cone are unstable. The next theorem proves this.
\begin{thm} \label{thm:stablity}
If a certain twist on the boundary of the motion cone results in a stable fixed-point, then the fixed-point resulted from the opposite twist is unstable.
\end{thm}
\begin{proof}
At a fixed-point of the system, the patch will not displace with respect to the object. Accordingly, we first derive the dynamics of the relative pose $\vc{q}_{rel} = \mx{R}(-\theta_o) \left(\vc{q}_{h} - \vc{q}_{o}\right)$. By taking the derivative of the expression, we find
\begin{align}
	\dot{\vc{q}}_{rel} &= \mx{G}^{-1} \vc{\nu}_{rel} \nonumber \\
		&=  \mx{G}^{-1} \left(\mx{I} + (\alpha \mx{B} \hat{\mx{A}}^{-1})^{-1} \right)^{-1} \vc{\nu}_h.
\end{align}
We linearize the system at the boundary of pivoting and sticking modes, i.e., at $\vc{q}_{rel} = \vc{q}_0$ where $\alpha = 0$.
The eigenvalues of the linearized system will tell us about stability of fixed-points when the system is perturbed away from the sticking region. Denoting the variation in $\vc{q}_{rel}$ by $\Delta \vc{q}_{rel}$, the linearized system can be written as
\begin{align}
	\dot{\Delta \vc{q}}_{rel} = \mx{M} {\Delta \vc{q}}_{rel},
\end{align}
where
\begin{align}
	\mx{M} =  \mx{G}^{-1} (\vc{q}_0) \mx{B}\hat{\mx{A}}^{-1} \vc{\nu}_h \nabla \alpha.
\end{align}
Here, we have considered $\alpha = \alpha(x_r,y_r,\theta_r)$, which is the solution of~\eqref{eq:solAlpha}.

First, we observe that $\mx{M}$ is a matrix of rank 1, since its columns are scaled versions of the same vector. 
Accordingly, $\mx{M}$ has at most one nonzero eigenvalue, which can be expressed as
\begin{align}
	\vc{u} := \mx{G}^{-1} \mx{B} \hat{\mx{A}}^{-1} \vc{\nu}_h, \label{eq:eigvM}
\end{align}
and the corresponding eigenvalue is 
\begin{align}
	 \nabla \alpha \vc{u} = [\pder{\alpha}{x_r},\,\pder{\alpha}{y_r},\,\pder{\alpha}{\theta_r}] \vc{u}. \label{eq:eigM}
\end{align}

Note that $\alpha$ is an even function of $\vc{\nu}_h$, i.e., the solution of~\eqref{eq:solAlpha} is the same for both~$\vc{\nu}_h$ and $-\vc{\nu}_h$. Hence, the gradient of $\alpha$ is unaffected by changing the direction of motion while $\vc{u}$ according to~\eqref{eq:eigvM} will be negated. Therefore, we conclude that for opposite twists, the signs of the eigenvalues of the linearized system are different. This implies that if a fixed-point resulted from a twist on the boundary of pivoting and slipping regions is stable, the fixed-point due to the opposite twist cannot be stable.
\end{proof}
As it can be concluded from the proof of Theorem~\ref{thm:stablity}, the relative position does not have an asymptotic stability. Thus, the patch can gradually drift with respect to the object unless the system is forced into sticking mode with some margin.
Moreover, to find if a direction of motion is stable, we have to check whether the sign of expression~\eqref{eq:eigM} is negative or not. Although it is possible to derive an analytic expression, a practical way is to check what direction the object would rotate toward if it is perturbed away from the fixed-point. For linear motions this approach works even if the object is far from the fixed-point, since the object rotates always toward a stable fixed-point.

%% file: 05_approxSol.tex
\begin{figure*}
	\begin{center}
		\includegraphics[height=5cm]{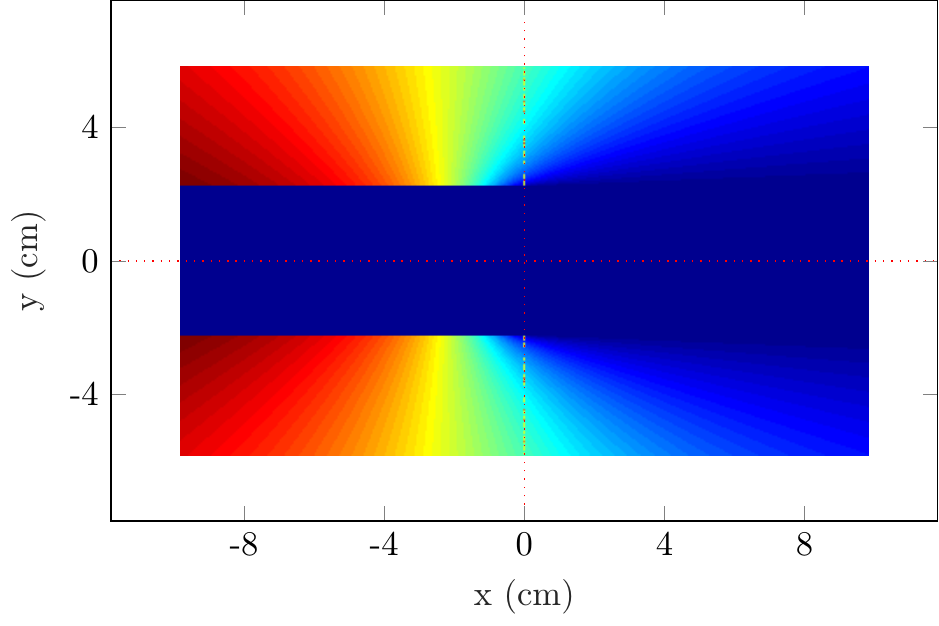}~
		\includegraphics[height=5cm]{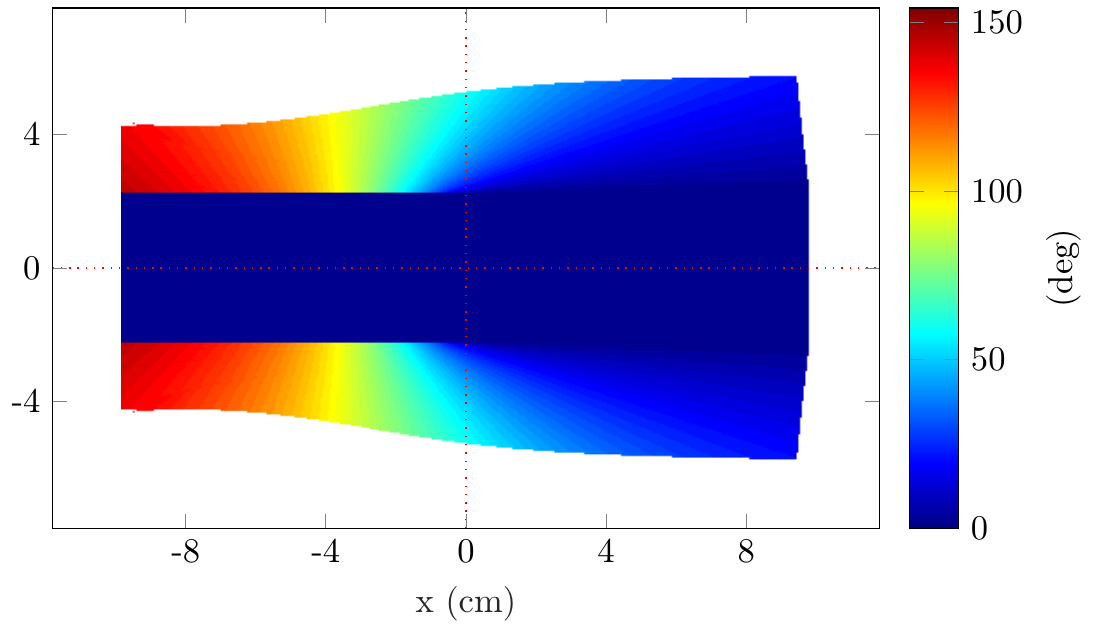}
		\caption{Visualization of maximum amount of rotation as a function of initial patch location for $\vc{\nu}_h = [1,\, 0,\, 0]^T$. On the left, the approximate solution based on the angle between the direction of hand motion and the stable boundary of the motion cone, on the right the final amount of rotation based on simulated sliding. If the hand is initially placed in the white area, the patch will touch an edge of the object or goes outside of the object boundary. In the dark blue area in the middle, the hand and the object stick together.}
		\label{fig:maxRot}
	\end{center}
\end{figure*}
In this section, we develop an approximate solution to the problem of determining the velocity of the object in pivoting mode assuming that the variation in the pivot point is limited, e.g., when the torsional friction of the patch is small. The approximate solution has to solve the same set of equations as before, i.e.,~\eqref{eq:intersec} and~\eqref{eq:oVelForceT}--\eqref{eq:VrelT}, but we relax some of the constraints assuming a known estimate of the pivot point. In practice, having such an estimate is realistic, since it is desired to keep the pivot point close the center of the patch. Moreover, we can devise an iterative algorithm to update the position of the pivot point in order to improve the accuracy of the solution.

The condition for the pivot point can be expressed as
\begin{align}
    \mx{J}_p \vc{\nu}_{rel} = 0, \label{eq:pivotPoint}
\end{align}
where $\mx{J}_p \in \mathbb{R}^{2 \times 3}$ corresponds to the first two rows of $\mx{J}(\vc{p})$, which is the Jacobian of the position of the presumed pivot point $\vc{p}$ defined in the hand coordinate.
Using the definition of $\vc{\nu}_{rel}$ in~\eqref{eq:VrelT}, combined with~\eqref{eq:oVelForceT}, we find 
\begin{align}
\mx{J}_p \hat{\mx{A}} \mx{\Phi} \mx{w} =  -\dfrac{\vc{v}_p}{k_1}, \label{eq:ppConst}
\end{align}
where we have defined $\vc{v}_p:= \mx{J}_p \vc{\nu}_h$.

Considering~\eqref{eq:ppConst}, the general solution for $\vc{w}$ given $\vc{v}_p$ is
\begin{align}
\vc{w} =  -\dfrac{1}{k_1} \mx{\Phi}^T \left(\mx{J}_p ^\dagger \vc{v}_p + \gamma \vc{n} \right), \label{eq:ppConstSol}
\end{align}
where $\vc{n} =[y_p,\, -x_p,\, 1]^T$ spans the nullspace of $\mx{J}_p$ and $\gamma \in \mathbb{R} $ is an arbitrary constant. By introducing $\bar{\vc{w}}: = \mx{\Phi}^T \mx{J}_p^\dagger \vc{v}_p$ and $\vc{w}_0 = \mx{\Phi}^T \vc{n}$, we rewrite~\eqref{eq:ppConstSol} as
\begin{align}
\vc{w} =  -\dfrac{1}{k_1} \left( \bar{\vc{w}} + \gamma \vc{w}_0 \right). \label{eq:sliderW}
\end{align}
Note that $\mx{J}_p$ is a function of the pivot point, and hence $\vc{w}$. 

Substituting~\eqref{eq:sliderW} into~\eqref{eq:intersec1} results in
\begin{align}
    \gamma^2 \vc{w}_0^T \mx{C} \vc{w}_0 + 2 \gamma  \vc{w}^T_0 \mx{C} \bar{\vc{w}} + \bar{\vc{w}}^T \mx{C} \bar{\vc{w}} = 0, \label{eq:gamma}
\end{align}
where we have used the facts that $k_1 > 0$ and $\mx{C}$ is diagonal.
From~\eqref{eq:gamma}, we can solve for 
\begin{align}
    \gamma =  \dfrac{-\vc{w}^T_0 \mx{C} \bar{\vc{w}} \pm \sqrt{\Delta}}{\vc{w}_0^T \mx{C} \vc{w}_0}, \label{eq:solAlphaOld}
\end{align}
where
\begin{align}
    \Delta := (\vc{w}^T_0 \mx{C} \bar{\vc{w}})^2 - (\vc{w}_0^T \mx{C} \vc{w}_0)(\bar{\vc{w}}^T \mx{C} \bar{\vc{w}}).
\end{align}
After determining $\gamma$, we can substitute~\eqref{eq:sliderW} into~\eqref{eq:intersec2} to find $k_1$. Consequently,
\begin{align}
    k_1 = \norm{\bar{\vc{w}} + \gamma \vc{w}_0}. \label{eq:apxK1}
\end{align}

Moreover, irrespective of $k_1$, by putting $\vc{w}$ back to~\eqref{eq:oVelForceT} it is possible to conclude
\begin{align}
\vc{\nu}_o =  \mx{G}^{-1} \left(\mx{J}^\dagger_p \mx{J}_p \vc{\nu}_h + \gamma \vc{n} \right), \label{eq:voAppx}
\end{align}
where the pivot point and $\gamma$ depend on the friction parameters and $\vc{\nu}_h$. 
Also using~\eqref{eq:Vrel}, we obtain
\begin{align}
	\omega_r  = \dfrac{\vc{n}^T \vc{\nu}_h}{\norm{\vc{n}}^2} + \gamma.
\end{align}

So far, we have ignored Equation~\eqref{eq:rVelForceT}. From this equation, we find
\begin{align}
	k_2  = - \dfrac{\omega_r}{[0\,, 0\,, 1]  \mx{B} \mx{\Phi} \vc{w}}. \label{eq:k2Appx}
\end{align}
There are two solutions to~\eqref{eq:gamma}. Using Equation~\eqref{eq:k2Appx}, we choose the one which results in a positive value for $k_2$.
Now, we can calculate $\vc{\nu}_{rel}$ from~\eqref{eq:rVelForceT} and use~\eqref{eq:ppoint} to update the pivot point. 

Note that by choosing the frame of the hand at the pivot point, the relation simplifies to
\begin{align}
	\vc{\nu}_o =  \mx{G}^{-1}
	\begin{bmatrix} v_{hx}\\ v_{hy}\\ \gamma \end{bmatrix}.
\end{align}
In this case, the angular velocity of the object is equal to $\gamma$.

Other possibility that we can study using this formulation is when there is a fixed amount or no torsional friction. In this case,~\eqref{eq:intersec1} must be replaced by
\begin{align}
	\vc{w}_0^T \vc{w} = m_p, \label{eq:mpFixed}
\end{align}
which is a consequence of Proposition~\ref{thm:trans}. 
Subsequently, we substitute~\eqref{eq:sliderW} into~\eqref{eq:mpFixed} to get
\begin{align}
\gamma \vc{w}_0^T \vc{w}_0  = -\vc{w}_0^T\bar{\vc{w}} -k_1 m_p. \label{eq:gamma2}
\end{align}
If $m_p \neq 0$, Equation~\eqref{eq:gamma2} can be solved together with~\eqref{eq:apxK1} which results in a quadratic equation. However, if $m_p = 0$ we find immediately
\begin{align}
\gamma  = - \dfrac{\vc{w}_0^T\bar{\vc{w}}}{\norm{\vc{w}_0}^2} = \dfrac{\vc{n}^T \hat{\mx{A}}^{-1} \mx{J}^\dagger_p \vc{v}_p}{\vc{n}^T \hat{\mx{A}}^{-1} \vc{n}}.
\end{align}
This result reveals the connection to pulling/pushing scenarios in which there is no torsional friction at the pivot point.

%% file: 06_movePrim.tex
In this section, we consider an application of the model developed in previous sections. Specifically, we are interested in using a robotic soft finger to reorient a flat object on a table. This can be done similarly to how a human would manipulate a cell phone on a table, as illustrated in Figure.~\ref{fig:slidingCellphone}.

When a part of the hand is used for reorienting an object, we leverage on the fact that there is a pivot point in the vicinity of the contact region. The instantaneous pivot point is the point where the patch and the object have the same velocity, i.e., the relative velocity is zero. Note that a pivot point at infinity is not interesting since it implies that the patch slides completely against the object, hence cannot move it. See Figure~\ref{fig:estParamsPath} for an example of the locus of the pivot point and the origin of the hand frame during sliding motion.

\begin{figure*}
	\begin{center}
		\includegraphics[width=0.48\linewidth]{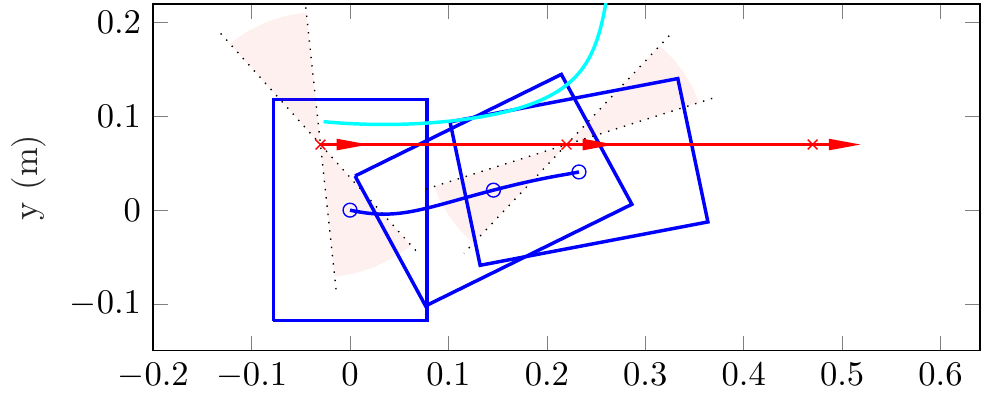}~
		\includegraphics[width=0.45\linewidth]{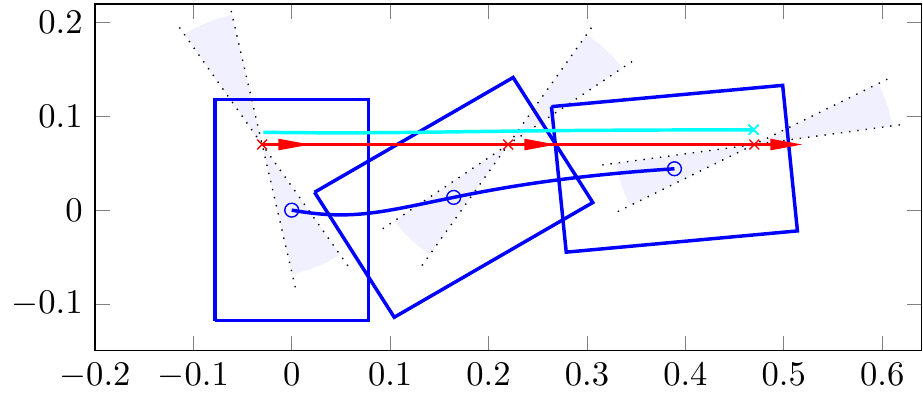}\vspace*{10pt}
		\includegraphics[width=0.48\linewidth]{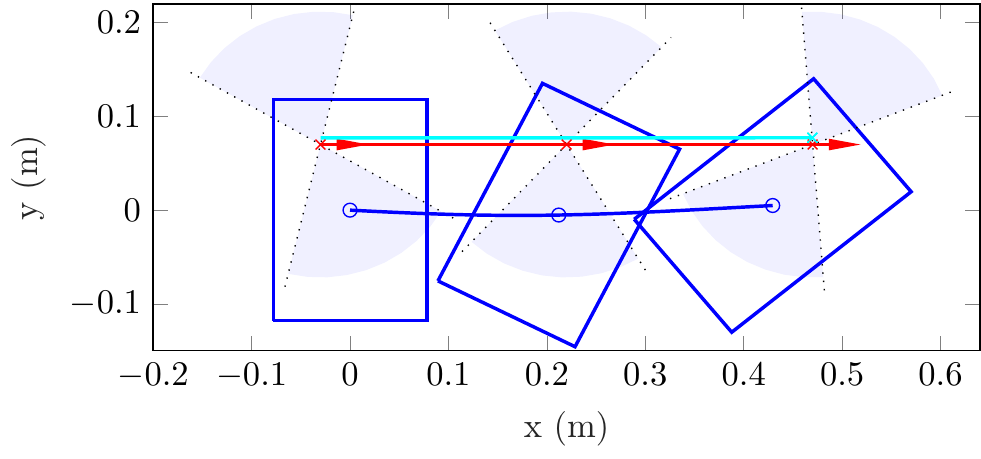}~
		\includegraphics[width=0.45\linewidth]{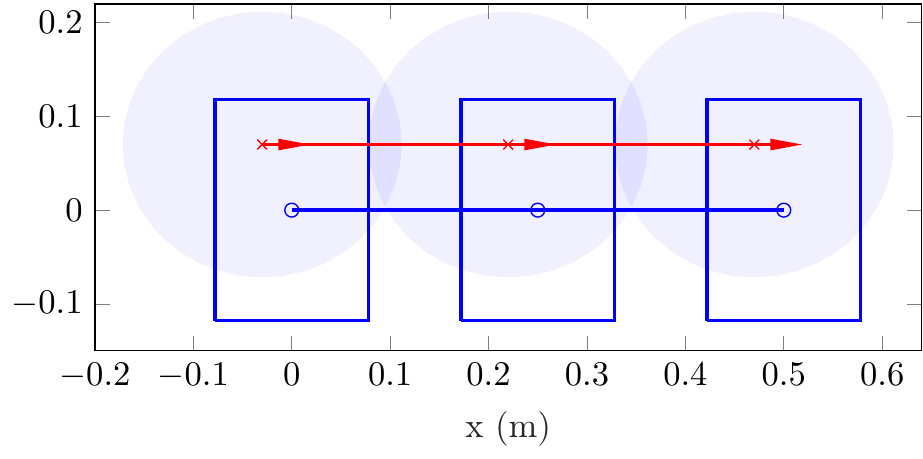}
		\caption{Sliding an object along a straight line from left to right for several fixed normal forces. The origin of the object frame, the origin of the hand frame and the pivot point are shown in blue, red, and cyan, respectively. The dashed lines represent the boundary of motion cones. The slipping and sticking cones are shaded in light red and light blue, respectively. In the top-left figure, the hand slips after some time completely outside the boundary of the object, while in the bottom-right figure the mode is always sticking because of the increased normal force.}
		\label{fig:slidingfn}
	\end{center}
\end{figure*}

For sliding, the placement of the hand is constrained by 
\begin{itemize}
    \item the direction and amount of rotation,
    \item final constraints such as not causing the object to topple at the edge of the surface,
    \item the weight of the object to avoid flipping it,
    \item the requirement for pivoting, i.e., the object should not be caged.
\end{itemize}
If the trajectory of the patch and its final relative pose with respect to the object are given, then there is not much freedom left for the initial placement of the hand. Otherwise, the placement can be optimized for different criteria, such as the shortest distance to achieve a certain rotation or the lowest uncertainty, e.g., by ensuring that the desired pose is a stable fixed-point.

When using a linear motion to reorient an object, the upper limit for the amount of object rotation is the angle between the velocity vector of the patch and the line passing through the COPs of the patch and the object, provided that the friction is isotropic and there is no torsional friction~\citep{huang2017exact}. A more accurate estimation can be obtained by measuring the angle between the velocity vector and the stable boundary of the planar motion cone for a given hand placement. Figure~\ref{fig:maxRot} visualizes the amount of expected rotation as a function of the initial hand placement from simulation and compares it with the approximation based on motion cones. 

If the goal is to rotate the object to a desired orientation on the edge of a table, a possible strategy could be to place the hand on the side where a straight trajectory gives a feasible solution for a long enough surface. The friction patch should be close to the edge to shorten the distance, but not exactly at the edge to allow for pressing the object downward even when the object has reached the edge of the surface. The placement should also leave some margin to the borders of the object since the relative position of the patch and the object may vary during the motion.

From the dynamical system described in Section~\ref{sec:dynSys}, we know that only two out of the three following quantities can simultaneously be controlled:
\begin{enumerate*}[label=(\roman*)]
	\item normal force,
	\item velocity of the hand,
	\item the position of the pivot point.
\end{enumerate*}
Moreover, possible feedback signals for a controller are $\vc{w}_h$, $\vc{\nu}_o$, $\vc{q}_o$, and $\vc{q}_{rel}$. 

For example, we may keep the normal force at a minimum amount required for stopping rotation. Then, using the feedback from the actual angle of the object, the normal force can be decreased if the angular velocity has to be increased.
Accordingly, a simple proportional controller can be designed as
 \begin{align}
	f_n =  -K \sat(\theta_{ref} - \theta_{meas}) + f_{U},
	\label{eq:controller}
 \end{align} 
where $\theta_{ref}$ is the reference angle and $f_{U}$ an upper limit for the normal force, which can be obtained from the model,  $\sat(\cdot)$ is the saturation function with tunable upper and lower limits, $K$ is the proportional gain. 

A lower bound for the normal force can be chosen such that there will be no slipping mode and a fixed-point, i.e., switching to the sticking mode is guaranteed. More advanced schemes might include an integral action to reduce the sensitivity to $f_{U}$, and a feedforward term.

%% file: 07_results.tex
We performed a number of simulations and robotic experiments to validate the model and its application. 
The simulations were carried out in Matlab, and the robotic experiments were performed by a Kuka LBR iiwa7 robot.

\subsection{Simulations}

\begin{figure}
	\begin{center}
		\includegraphics[width=1.0\linewidth]{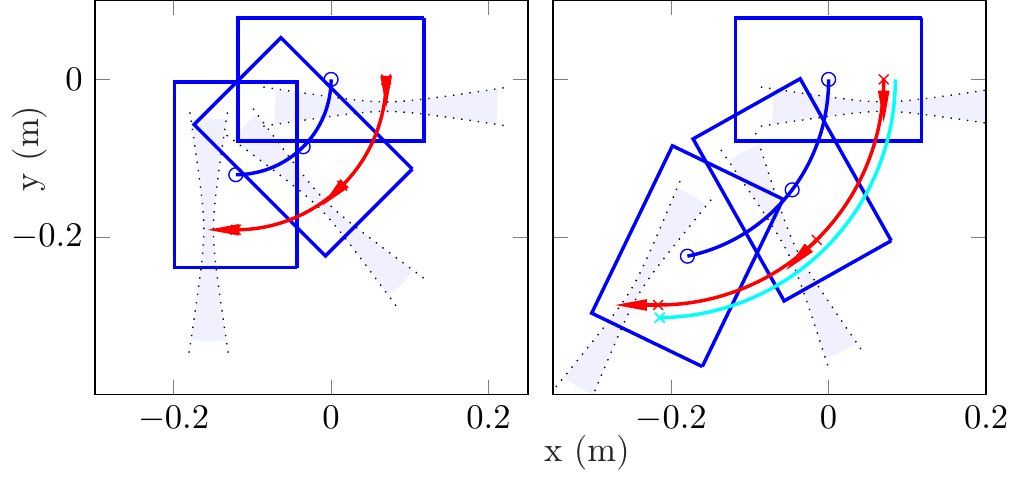}
		\caption{Sliding an object with $\omega_h \neq 0$. On the left, the tip of the velocity arrow is inside the motion cone, thus the patch sticks to the object. On the right side, the liner velocity has increased and the object is pivoting against the patch.}
		\label{fig:cones_w}
	\end{center}
\end{figure}

\begin{figure*}
	\begin{center}
		\includegraphics[width=\linewidth]{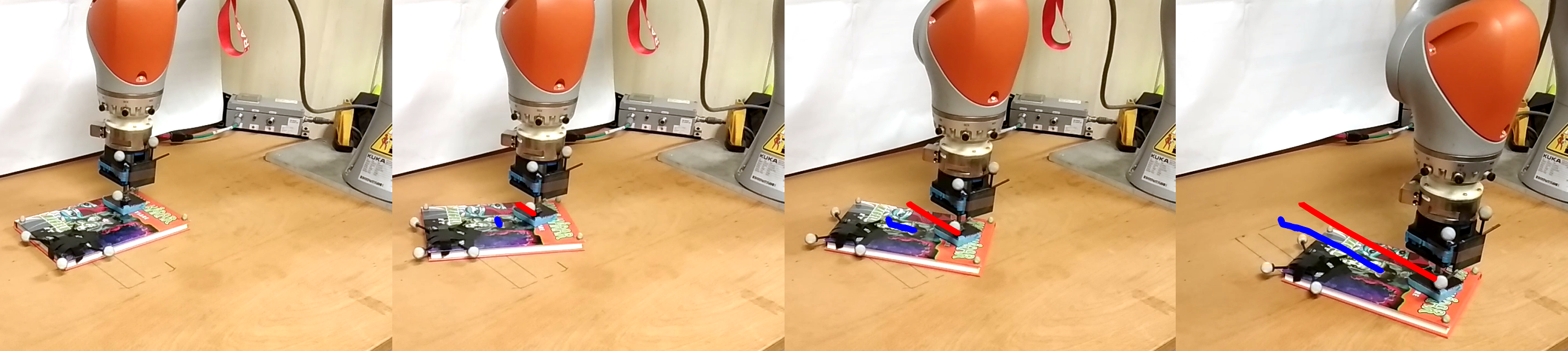}
		\caption{Sliding motion experiment: a soft finger is attached to the KUKA LBR iiwa robot. The book is being dragged toward the edge. Trajectories of the center of the object are shown in blue and of the friction patch in red.}
		\label{fig:sliding_exp}
	\end{center}
\end{figure*}
\begin{figure}
	\begin{center}
		\includegraphics[width=0.9\linewidth]{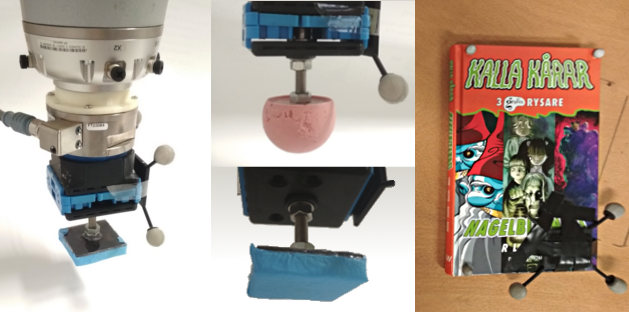}
		\caption{Experimental setup. From left to right: robot end-effector with optical markers, spherical and square soft fingers, and the object (book) with optical markers.}
		\label{fig:probes}
	\end{center}
\end{figure}

We present the results of simulations of the model with the parameters specified here.
The box dimensions are $15.6 \times 23.6\,$cm. The patch is circular with a radius of $2.0\,$cm. The weight of the box is $450\,$g. The coefficient of friction between the box and the surface and the soft finger and the box are $\mu_{oe} = 0.2$ and $\mu_{ho} = 0.8$, respectively. We assume a uniform pressure distribution between the box and the surface when the box is not pressed and a Hertzian pressure distribution for the soft finger. To account for the shift of COP, we use~\eqref{eq:COPmodel} with $c = 0.6$ and $\delta = 2$.

In Figure~\ref{fig:slidingfn}, simulation experiments in which the blue box is being moved from the left to the right by a soft finger are illustrated. The initial placement of the soft finger is $\vc{r} = [-3,\, 7]\,$cm. The end-effector moves at $1\,$cm/s in the $x$ direction. The simulation runs for $50\,$ seconds. Each subplot corresponds to a certain constant normal force. Considering left-right top-bottom ordering, the normal forces are $1.43$, $1.7$, $4$, and $6\,$N, respectively. The trajectory of the patch, pivot point, and object are visualized.
Note that for generating Figures~\ref{fig:hLocus},~\ref{fig:ppLocus}, and~\ref{fig:maxRot} presented in previous sections, we have used ${f_n = 2.5\,}$N, $\vc{r} = [-3.5,\, 6]\,$cm, and $f_n = 2.5\,$N, respectively while remaining parameters were set according to the values given in this subsection.

When the soft finger is moving with an angular velocity, the mode depends on the magnitude of the linear velocity. Thus, for the same normal force, angular velocity, and direction of velocity of the finger, different modes might arise. This is because of the fact that when $\omega_h \neq 0$, the modes are no longer mapped to planar cones. Instead, the boundaries of the regions in two dimensions can be represented by the intersection of the motion cone described in Theorem \ref{thm:coneReg} with the plane corresponding to the angular velocity $\omega_h$.
Figure \ref{fig:cones_w} shows an example where the finger is rotating at $-\pi/80$ rad/s for 40 seconds.
On the left side, the linear velocity is chosen so that it is within the sticking region. In this case, the object is rotated by 90$^\circ$. On the right, the linear velocity is slightly increased such that it enters the pivoting region, resulting in a faster rotation of the object, exceeding the 90$^\circ$ rotation by the end of the simulation.

\subsection{Robotic experiments}
The experimental setup consisted of a KUKA light weight iiwa7 robot arm, with an ATI Gamma force-torque sensor mounted at the wrist.
A number of soft fingers were manufactured, and are shown in Figure~\ref{fig:probes}, together with the end-effector of the robot and the object used in the experiments (a hard-cover book).
The positions of the robot and of the object were recorded using an Optitrack motion capture system.

\begin{figure*}[t]
	\begin{center}
		\includegraphics[width=1.0\linewidth]{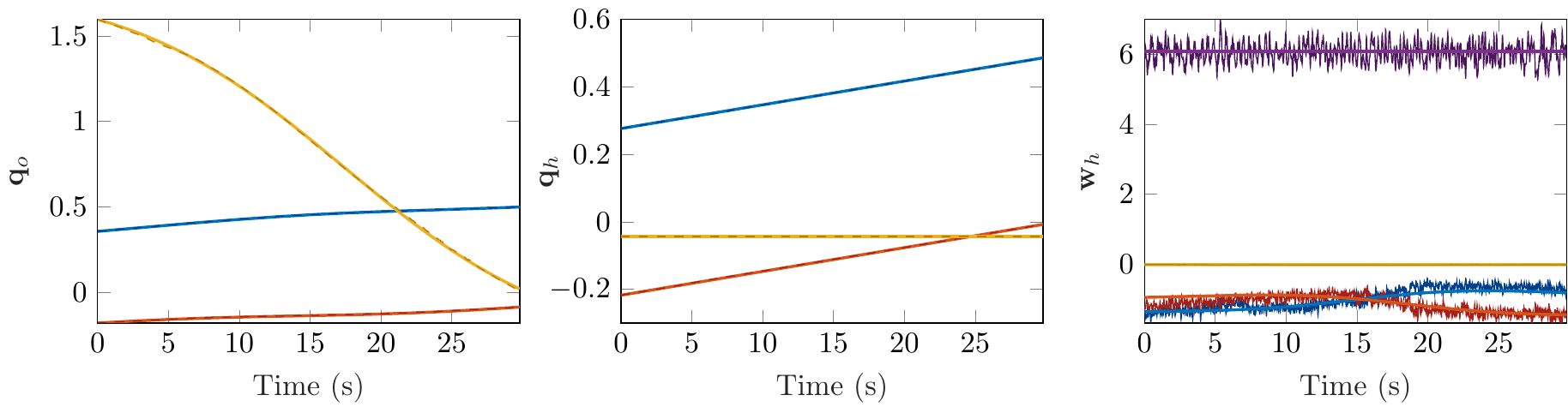}
		\caption{A sample of poses and forces in pivoting mode during a straight line motion. The three components of each vector are shown in blue, red, and yellow colors, respectively. The normal force at the hand frame has additionally been shown in $\vc{w}_h$ plot in violet. Dashed lines represent the experimental results, which are almost indistinguishable from the simulation.}
		\label{fig:estParamsAll}
	\end{center}
\end{figure*}
\begin{figure*}[t]
	\begin{center}
		\includegraphics[height=7.0cm]{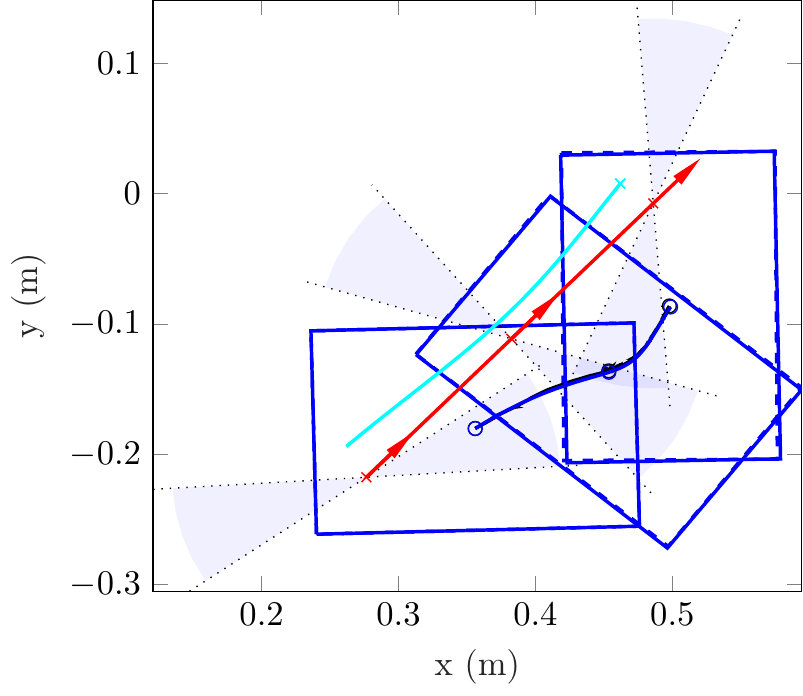}
		\includegraphics[height=7.0cm]{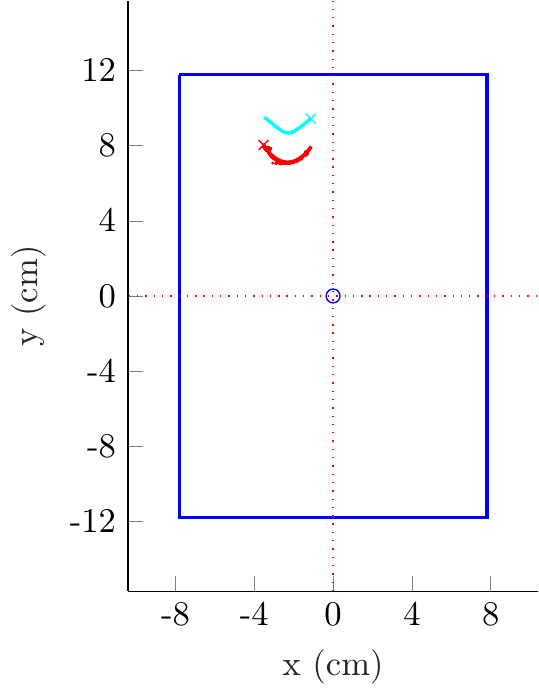}
		\caption{Comparison of simulated and experimental results for a straight line pivoting motion.  On the left, the rectangles in blue illustrate simulation, and in dashed black experimental results. On the right, the locus of the origin of the patch frame (red) and the pivot point (cyan) with respect to the object are shown.}
		\label{fig:estParamsPath}
	\end{center}
\end{figure*}

\begin{figure}
	\begin{center}
		\includegraphics[width=0.9\linewidth]{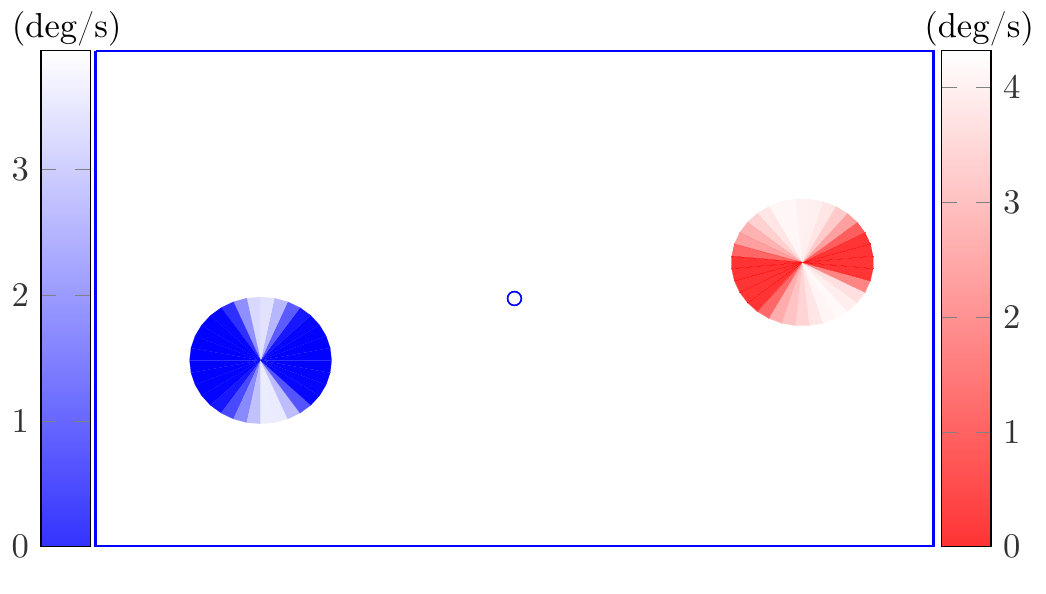}
		\caption{Visualization of the sticking and pivoting cones (blue) and slipping and pivoting cones (red) based on $\abs{\omega_o}$ as a result of moving the soft finger with $1\,$cm/s in various directions.}
		\label{fig:regions_exp}
	\end{center}
\end{figure}

\subsubsection{Sliding trajectory}

To verify the accuracy of the dynamical system presented in Section~\ref{sec:dynSys}, the soft finger mounted at the robot end-effector was pressed against the object and commanded to move at a certain velocity, while maintaining a constant normal force.
An image sequence of one trial is shown in Figure \ref{fig:sliding_exp}, with the trajectories of the center of the book overlaid in blue and of the center of the friction patch in red.

An experiment under similar conditions was tested in simulation, calculating the trajectories of the object, finger, and pivot point, and the forces that arise from this interaction.
In Figure~\ref{fig:estParamsAll}, the full state of the system and the wrenches at the hand frame as a function of time are shown for both simulation (solid) and experimental (dashed) data.
To identify $\mx{A}$, $\mx{B}$, and $s$, i.e., the percentage of the shift of COP due to loading, we set up an optimization problem that minimizes the error between the simulated experiment and the measured data. The hand velocity and the normal forces are chosen as the average of the respective values from the experiments.
A comparison between the simulated and the experimental results is shown in Figure~\ref{fig:estParamsPath}.
The plot on the left shows the simulated object path in blue and the experimental in dashed black.
The plot on the right side of Figure~\ref{fig:estParamsPath} shows the positions of the patch and the pivot point in the object frame.
It can be seen that the model accurately describes the sliding motion of the object, reaching similar positions and orientations within the same amount of time.
In Figure~\ref{fig:traj_examples}, a number of sample sliding motions are shown. The same parameters, except for the COP of the patch which could vary slightly from an experiment to another, are used for all the simulated results. The prediction of the proposed model matches with the experimental results.

\subsubsection{Modes and motion cones}
Validation of the motion cones and possible modes was carried out by placing the soft finger at different locations on the book and performing linear motions in various directions, with zero angular velocity, while maintaining a constant normal force.
The angular velocity of the object was recorded and is visualized in Figure~\ref{fig:regions_exp} for two different locations.
The symmetry presented in these results confirm what is posited in Corollary \ref{cor:revTwist}.
The left side of the figure shows the angular velocities when pressing the soft finger with a normal force of $6\,$N.
Since the friction is approximately isotropic, according to Theorem~\ref{thm:linMotionStick} the soft finger sticks to the object when moving towards or away from the COP of the object.
When moving perpendicularly to this line, the object pivots and has some rotational velocity.
The right side shows the same effect with a normal force of $2\,$N.
The patch slips against the object when moving along a direction close to the line that passes through the COPs of the object and of the patch and pivots when the velocity is perpendicular to that.

\subsubsection{Controlled sliding} \label{sec:contSliding}
One of the main observations of the proposed model is that, by regulating the normal force applied by the soft finger, we can modify the trajectory of the object.
As discussed in Section~\ref{sec:normalforce}, an increase in the normal force applied on the object through the friction patch slows down the rotation.
Given a patch location and a velocity direction, a reference trajectory for $\theta_o$ can be defined, as long as it stays within the calculated limits as in Figure~\ref{fig:maxRot}.

Figure \ref{fig:controlExp} illustrates the result of an experiment for tracking a desired trajectory of the object orientation.
A force controller was implemented on the robot to realize the proportional control law~\eqref{eq:controller}.
This simple proportional controller was able to closely track the reference trajectory, applying larger normal forces to keep the object from rotating, and relaxing the pressure whenever faster rotation was required.

\begin{figure*}
	\begin{center}
		\includegraphics[width=1.0\linewidth]{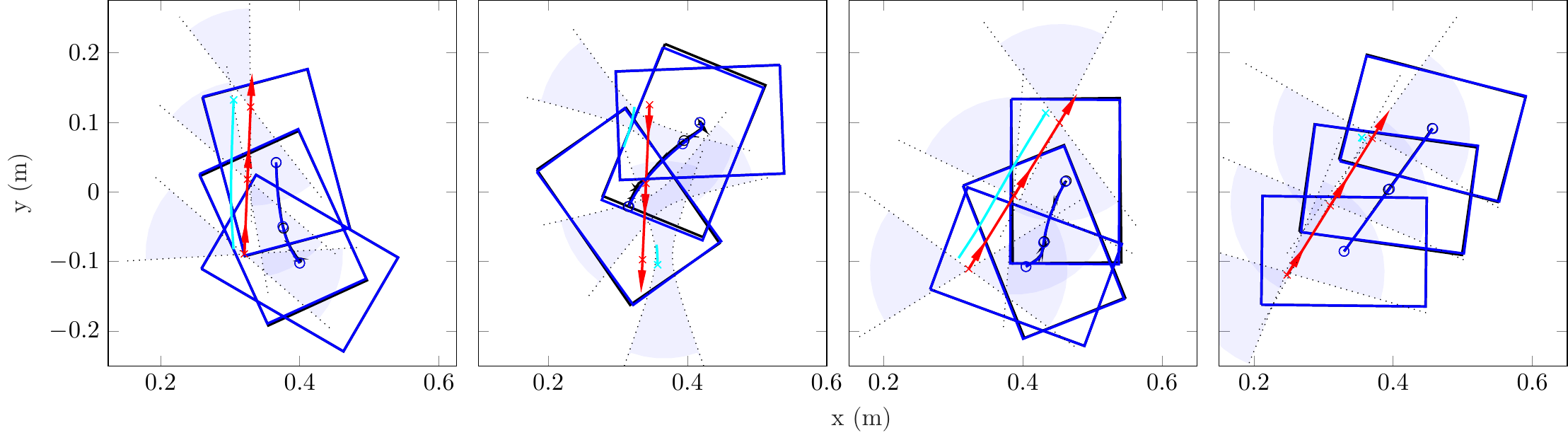}
		\caption{Samples of sliding motion illustrating paths and motion cones. Blue: modelled; Black: experimental. Red: friction patch; Cyan: pivot point.}
		\label{fig:traj_examples}
	\end{center}
\end{figure*}

\begin{figure}
	\begin{center}
		\includegraphics[width=0.9\columnwidth]{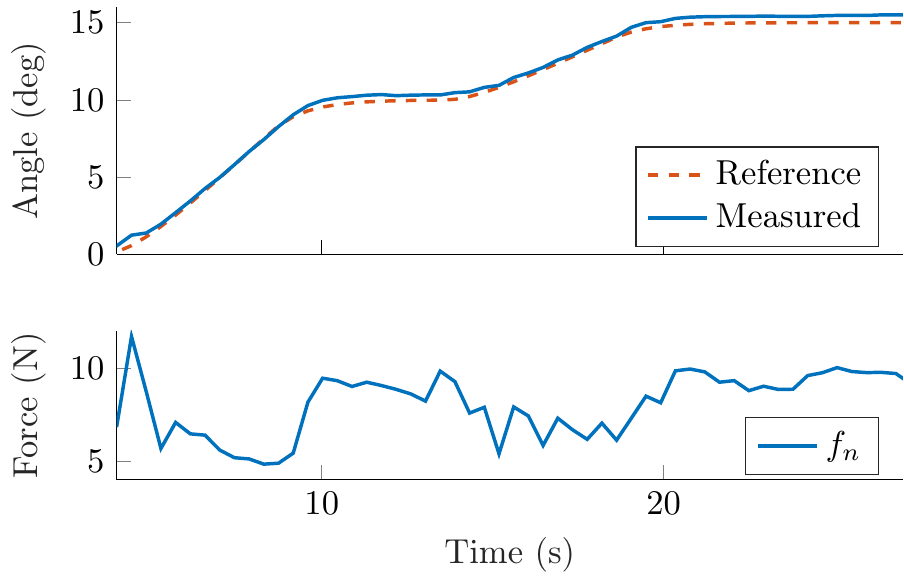}
		\caption{Experimental result of tracking a reference trajectory for the object orientation.}
		\label{fig:controlExp}
	\end{center}
\end{figure}

%% file: 08_disc.tex
The experimental results suggest a good match between theory and practice. The assumptions of quasi-static motion and ellipsoid approximation of limit surfaces hold in our experiments. However, when the contact areas are not negligible, an accurate estimation of the location of COPs becomes important. Accordingly, due to variability in the pressure distribution and friction coefficients~\citep{yu2016more}, applying nominal values may not be adequate and hence feedback must be employed.

A frictional patch allows both pulling and pushing of an object. Compared to pushing scenarios~\citep{mason1982manipulator,zhou2017pushing}, we observe that the torsional friction adds some stability margin for pushing. Nevertheless, reversing from a stable motion at the border of the pivoting region is unstable and sensitive to small perturbations.
Additionally, at the stable fixed-point, the center of a circular patch would not align with the COP of the object, in contrast to the scenario where there is no torsional friction considered by~\cite{huang2017exact}. Another observation is that, if the friction is anisotropic, the main axis of the motion cone corresponding to sticking mode is not directed toward the COP of the object.

A fundamental difference between our model and previous studies based on pulling/pushing is the inclusion of torsional friction. Since the torsional friction can be regulated by adjusting the normal force, we have in effect an extra degree of actuation besides the motion of the hand. This fact has been successfully utilized in the trajectory tracking experiment in Section~\ref{sec:contSliding}. The reason that a simple controller such as~\eqref{eq:controller} can work in practice is that for a given velocity of the patch, the angular velocity of the object ($\omega_o$) could be approximated as a linear function of $f_n$ around a working point.

For a linear motion and a fixed normal force, the distance to cause a certain amount of rotation in the object is independent of the velocity of the patch. Moreover, by increasing the normal force, the motion cone widens, and hence the maximum possible rotation is reduced. At the same time, the object rotates slower, which results in a longer distance to achieve a certain amount of rotation.

According to Theorem~\ref{thm:stickMaxMargin}, the generalized eigenvectors of $\hat{\mx{A}}$ and $\mx{B}$ can be interpreted as the principal sliding wrenches, while the generalized eigenvectors of $\hat{\mx{A}}^{-1}$ and $\mx{B}^{-1}$ define principal sliding twists. These eigenvectors and their corresponding eigenvalues provide important information for a sliding motion at a given configuration. Specifically, the eigenvectors characterize wrenches and twists that cause sticking and slipping with largest margins.

The approximation method developed in Section~\ref{sec:approxSol} can be viewed as an extension of the pushing model by~\cite{zhou2017pushing,zhou2017fast}, where the pusher can also exert torsional friction, and the ellipsoidal limit surfaces are given in their most generic form. If the pivot point is additionally allowed to displace, the model for planar sliding using friction patches is recovered. As explained in Section~\ref{sec:approxSol}, the displacement of the pivot point can be accounted for in the approximate solution in an iterative fashion.

If the stiffness of the material in contact with the slider is low, we cannot assume the velocity of the patch as a control input anymore. However, for many practical applications the deflection of the soft finger with respect to the overall motion of the object is negligible. A possible extension is to augment the Coulomb friction model of the patch with a spring-damper model to account for the flexibility of the soft finger.
This approach is somewhat similar to the work by~\cite{kao1992quasistatic}, although they do not derive any dynamical system for time evolution of the system.

A more realistic friction model should be able to distinguish between static and dynamic friction coefficients.
Moreover, higher order approximation of limit surface proposed by~\cite{zhou2018convex} can be adapted to our approach, although there may not exist efficient algorithms to solve the resulting systems of equation. 

A control algorithm can benefit from haptic exploration for identification of the friction properties of the object being manipulated. For example, by adjusting the normal force and measuring tangential forces, it is possible to distinguish if the soft finger slides \emph{with the object} or \emph{on the object}. Such ideas can be incorporated for identification and control of the mode of the system.

Using the developed model, several planning problems can be formulated and solved. A possible planning problem is to find trajectories to slide an object from an initial to a desired pose such the contact between surfaces is maintained and the normal force is within desired limits. Another example is to find minimal required normal forces to slide an object along a path without pivoting or slipping. Note that when the pivot point varies a lot, a simple approximation of the dynamics as an object with a joint is not sufficient. In such scenarios, a more elaborate dynamical system as the one presented in this article must be considered.

%% file: 09_conclusion.tex
A mathematical model for planar sliding has been provided, where the friction between a soft finger and the slider generates forces for the motion of the object.
The concept of motion cones are extended to find criteria for three modes of the system, namely sticking, slipping, and pivoting. In pivoting mode, the soft finger can be regarded as if it is connected to the slider using a joint, which in general slides on the surface of the object. This interconnection enables both pushing and pulling of the object.

We relax some of the common assumptions made for modeling frictional contacts such as diagonal matrices for approximating limit surfaces or having a fixed pivot point. The result of modeling is a hybrid dynamical system, which is used for finding fixed-points of the system and determining their stability.

The developed model can be used as the basis for planning and control design.
For example, it makes it possible to predict the directions for stable sliding with the largest margin, to approximate the amount of rotation before reaching a fixed-point,  to predict whether the contact is maintained or the soft finger would go off the surface of the object, etc.

We evaluate the model experimentally. The comparison of the results with the simulated model suggests that the essential physical factors are captured in the model. In addition, we demonstrate the possibility of tracking a desired trajectory by regulating normal forces based on visual feedback.

In the future, we consider designing  adaptive and robust controllers to deal with the variations in the parameters of the system during sliding.

%% file: 10_appendix.tex
\section{Proof of Theorem~\ref{thm:decomp}} \label{sec:appx2}

\begin{proof}
Since matrix $\mx{A}$ is symmetric, it has 6 unique elements. Accordingly, by expanding the right hand side of~\eqref{eq:decomp}, we find 6 equations and 6 unknowns, i.e., the diagonal elements of $\mx{\Lambda}$, the angle of rotation $\theta$  for $\mx{R} := \mx{R}(\theta)$ and the vector  $\vc{r} = [x_r,\, y_r]^T$ for the Jacobian $\mx{J} := \mx{J}(\vc{r})$. To show that this equation system can indeed be solved, we construct the solution.

Using the elements of $\mx{A}$, we calculate
\begin{subequations}
\begin{align}
	x &:=  (a_{11} - a_{22} ) + \dfrac{1}{a_{33}} (a_{23}^2 - a_{13}^2 ) \nonumber \\ 
	   &\phantom{:}=  (\lambda_{1} - \lambda_{2}) \cos(2\theta), \\
	y &:= -2 (a_{21}  - \dfrac{1}{a_{33}} a_{13} a_{23}) \nonumber \\
	   &\phantom{:}=  (\lambda_{1} - \lambda_{2}) \sin(2\theta).
\end{align}
\end{subequations}
Accordingly, we find
\begin{align}
	\theta  = \dfrac{1}{2} \Atan2(x,y) \sgn (\lambda_{1} - \lambda_{2}).
\end{align}
If $\lambda_{1}  \neq \lambda_{2}$, and we decide a specific ordering for the elements of $\mx{\Lambda}$, e.g., $\lambda_{1}  > \lambda_{2}$, the angle is uniquely determined. If $\lambda_{1} = \lambda_{2}$, any angle can be chosen including $0$.

After finding $\theta$, as an intermediate step we calculate 
\begin{align}
	\tilde{\mx{\Lambda}} = \mx{R}(\theta) \mx{A} \mx{R}^T(\theta).
\end{align}
Now, the elements of $\vc{r}$ are obtained as
\begin{align}
	x_r = \dfrac{\tilde{\Lambda}_{23}}{\tilde{\Lambda}_{33}}, \quad
	y_r = - \dfrac{\tilde{\Lambda}_{13}}{\tilde{\Lambda}_{33}} 
\end{align}
Finally, 
\begin{align}
	\mx{\Lambda} = \mx{J}(-\vc{r}) \tilde{\mx{\Lambda}} \mx{J}^T(-\vc{r}).
\end{align}

\end{proof}

\section{Proof of Theorem~\ref{thm:no2alpha}} \label{sec:appx0}

\begin{proof}
	We prove the theorem by contradiction. Assume that there are two distinct positive solutions $\alpha_2 > \alpha_1 > 0$. We prove that this implies that $\vc{\nu}_h = 0$.
	
	Since $\hat{\mx{A}}$ and $\mx{B}$ are positive definite, so is $\mx{\Lambda}$, and hence all $\lambda_i> 0,\, i\in \{1,2,3\}$. 
	The coefficients $c_i$, which are the diagonal of $\mx{C} = \mx{\Lambda} - \mx{I}$, cannot all have the same sign. Otherwise, the left hand side of~\eqref{eq:solAlpha} will be irrespective of $\vc{\nu}_h$ either positive or negative and there will be no solution to the equation. This implies that one or two eigenvalues are less than one, while the rest are/is larger than one.
	Here, we consider the case where $\lambda_1 > 1 > \lambda_2 > \lambda_3$. Other cases are proven similarly.
	
	For $\alpha_i, i\in \{1,2\}$, we have
	\begin{multline}
	c_1 \left(\dfrac{\tilde{v}_{xh}}{\alpha_i \lambda_1 + 1} \right)^2 +
	c_2 \left(\dfrac{\tilde{v}_{yh}}{\alpha_i \lambda_2 + 1} \right)^2 +
	c_3 \left(\dfrac{\tilde{\omega}_{h}}{\alpha_i \lambda_3 + 1} \right)^2 \\ = 0.\nonumber
	\end{multline}
	We multiply the equation associated with $\alpha_i$ by
	\begin{align}
	 \left(\alpha_i \lambda_1  +1\right)^2 \label{eq:elimL1}
	\end{align}
	and subtract the resulting equations from each other to eliminate the first term. Accordingly,
	\begin{multline}
		c_2 \tilde{v}^2_{yh} \left(f_{1,2}(\alpha_2)  - f_{1,2}(\alpha_1) \right) + \\ 
		c_3 \tilde{\omega}^2_{h} \left(f_{1,3}(\alpha_2)   - f_{1,3}(\alpha_1) \right) = 0 \label{eq:noc1Term}
	\end{multline}
	where
	\begin{align} 
	f_{i,j}(\alpha) := \left(\dfrac{\alpha \lambda_i + 1}{\alpha \lambda_j + 1}\right)^2. \label{eq:fija}
	\end{align}
	For $\alpha \geq 0$ and positive values of $\lambda_i$, from the derivate of~\eqref{eq:fija} we find out that the function is monotonically increasing or decreasing, depending on the sign of $\lambda_i - \lambda_j$.
	Thus, given our assumptions about $\lambda_i$, both $f_{1,2}$ and $f_{1,3}$ are increasing functions. Taking this fact into account, the left hand side of~\eqref{eq:noc1Term} is always negative unless both $\tilde{v}_{yh}$ and $\tilde{\omega}_{h}$ are zero. If this is the case, from~\eqref{eq:solAlpha} we conclude that $\tilde{v}_{xh}$ must also be zero. Since $\vc{\nu}_h$ is assumed to be non zero, this completes the proof by contradiction.
\end{proof}

\section{Proof of Theorem~\ref{thm:unique}} \label{sec:appx}

\begin{proof}
Firstly, we show that it is impossible to have no active modes, i.e., there is at least one active mode. Secondly, we prove that it is impossible to have any two modes active at the same unless $\vc{\nu}_h = \vc{0}$.

Let us define 
\begin{align} \label{eq:fofab}
f(\beta,\gamma):= \vc{w}(\beta, \gamma)^T \mx{C} \vc{w}(\beta, \gamma)
\end{align}
where $\vc{w}(\beta, \gamma) = \left(\beta \mx{I} + \gamma \mx{\Lambda}\right)^{-1} \tilde{\vc{\nu}}_h$. 
If neither sticking mode nor slipping mode is possible, from conditions~\eqref{eq:condStick} and~\eqref{eq:condSlip}, we conclude
\begin{subequations} \label{eq:noStnoSl}
\begin{align}
f(\beta,0) &> 0, \\
f(0,\gamma) &< 0.
\end{align}
\end{subequations}
Also define $g(\alpha):= f(1,\alpha)$.
According to~\eqref{eq:fofab}, $\sgn f(\beta, \gamma) = \sgn g({\gamma}/{\beta})$. Consequently, conditions~\eqref{eq:noStnoSl} can be written as
\begin{align*}
g(0) > 0
\end{align*}
and for a large enough $\alpha$
\begin{align*}
g(\alpha) < 0.
\end{align*}
Since $g(\cdot)$ is a continuous function, there must exist an $\alpha > 0$ such that $g(\alpha)= 0$, i.e., there is a solution in pivoting mode. Therefore, it is impossible to have no active mode.

For ease of reference, here we summarize~\eqref{eq:condStick},~\eqref{eq:alpha4vh}, and~\eqref{eq:condSlip}, which provide the criteria for sticking, pivoting, and slipping modes, respectively
\begin{subequations} \label{eq:stickPivotSlip}
\begin{align}
\tilde{\vc{\nu}}_h^T \mx{C} \tilde{\vc{\nu}}_h &\leq 0, \label{eq:cond1}\\
\alpha >0,\quad \tilde{\vc{\nu}}_h^T \mx{C} (\mx{I} + \alpha \mx{\Lambda})^{-2} \tilde{\vc{\nu}}_h &= 0, \label{eq:cond2}\\\
\tilde{\vc{\nu}}_h^T \mx{C}  \mx{\Lambda}^{-2}  \tilde{\vc{\nu}}_h &\geq 0. \label{eq:cond3}\
\end{align}
\end{subequations}
Now assume that any pair of these conditions hold true. We can show that this results in a contradiction unless $\vc{\nu}_h = \vc{0}$. The proof construction is similar to the proof of Theorem~\ref{thm:no2alpha} given in Appendix~\ref{sec:appx0}. More specifically, the term corresponding to the largest eigenvalue is eliminated by multiplying the expressions by proper coefficients similar to~\eqref{eq:elimL1}. Here, we provide the details only for the case where~\eqref{eq:cond1} and~\eqref{eq:cond3} are assumed true.

We know that the diagonal elements of $\mx{C} = \mx{\Lambda} - \mx{I}$ cannot have the same sign. Let us assume $\lambda_1 > 1 > \lambda_2 > \lambda_3 > 0$. Accordingly,
\begin{multline} \label{eq:slip-stick}
\lambda_1^2 \tilde{\vc{\nu}}_h^T \mx{C}  \mx{\Lambda}^{-2}  \tilde{\vc{\nu}}_h - \tilde{\vc{\nu}}_h^T \mx{C} \tilde{\vc{\nu}}_h  = \\
c_2 \tilde{v}^2_{yh} \Big(\left(\tfrac{\lambda_1}{\lambda_2}\right)^2   - 1 \Big) +
c_3 \tilde{\omega}^2_{h} \Big(\left(\tfrac{\lambda_1}{\lambda_3}\right)^2   - 1 \Big) \leq 0
\end{multline}
Unless $\tilde{v}_{yh}$ and $\tilde{\omega}_{h}$ are zero,~\eqref{eq:slip-stick} is strictly negative. However, to fulfill~\eqref{eq:cond1} if $\tilde{v}_{yh} = \tilde{\omega}_{h} = 0$, $\vc{\nu}_h$ must also be zero since $c_1 > 0$. Thus, we conclude that if $\vc{\nu}_h \neq 0$, then
\begin{align}
\tilde{\vc{\nu}}_h^T \mx{C} \tilde{\vc{\nu}}_h  > \lambda_1^2 \tilde{\vc{\nu}}_h^T \mx{C}  \mx{\Lambda}^{-2}  \tilde{\vc{\nu}}_h,
\end{align}
which contradicts the assumption that the left hand side is less than or equal to zero and the right hand side is greater or equal to zero. Other scenarios for $\lambda_i$ are proven similarly.
\end{proof}


\section{Experimental validation of shift of COP} \label{sec:appx3}

\begin{figure}
	\centering
	\includegraphics[width=0.9\columnwidth]{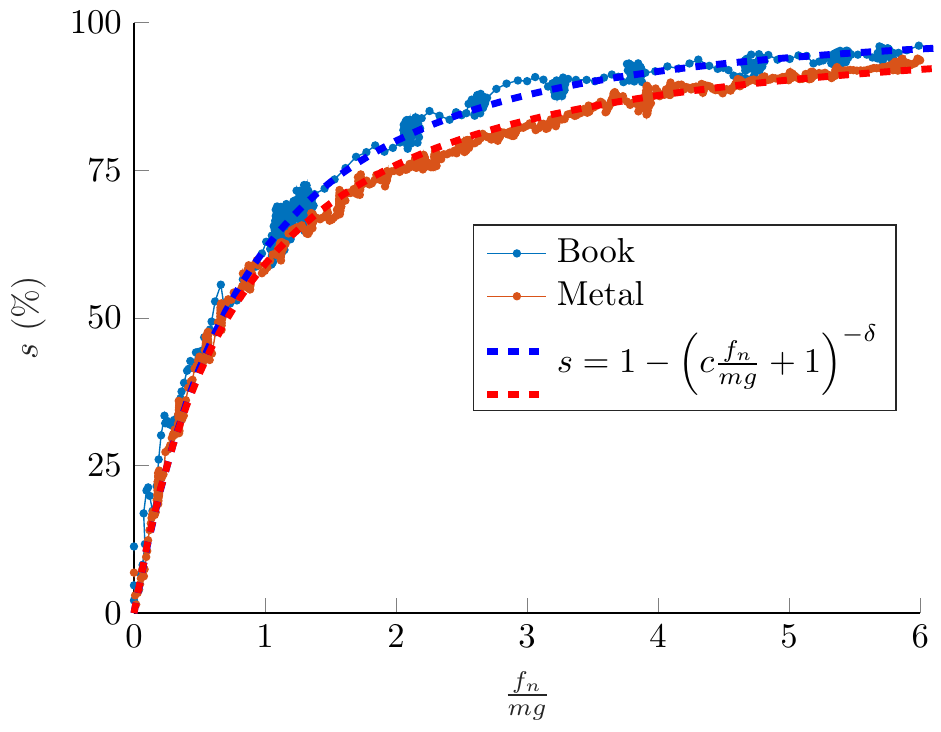}
	\caption{Experimental results of the effect of normal force in the shift of the object COP towards the patch.}
	\label{fig:copExp}
\end{figure}

To understand the effect of loading an object with a given normal force, in terms of the amount of displacement of the COP of the object towards the COP of the patch, a number of experiments were carried out.
We used a BTS Force Plate, which measures forces and centers of pressure. The objects were placed on the surface and pressed with an increasing normal force. 
The shift $s$ (in percentage) is plotted in Figure~\ref{fig:copExp}, against the normal force (normalized for object weight) for two different objects: a hardcover book of 463$\,$g, and a flat steel slab of 1593$\,$g.
Both objects presented similar behaviors, despite the differences in material properties.
The computational model proposed in \eqref{eq:COPmodel} was used to fit the experimental data, and the resulting parameters were ${c=0.6}$, $\delta=2.0$ for the book and $c = 0.9642$, $\delta = 1.324$ for the metal slab.